\documentclass[twoside,11pt]{article}
\usepackage{jmlr2e}

\usepackage{enumerate}
\usepackage{epsf}
\usepackage{epsfig}
\usepackage{amsmath}
\usepackage{subfigure}
\usepackage{amssymb}
\usepackage{algorithm}
\usepackage{multirow}
\usepackage{multicol}
\usepackage{array}
\usepackage{framed}
\usepackage{bm}
\def\srk{\mathsf{r}}
\usepackage{float}
\usepackage{lipsum}
\usepackage[framemethod=tikz,innertopmargin=7pt]{mdframed}
\usepackage{mathtools}
\usepackage{manfnt}
\usepackage{mathrsfs} 
\usepackage[symbol]{footmisc}
\usepackage{graphicx} 
\usepackage{ltxtable} 
\newcolumntype{L}{>{\centering\arraybackslash} m{0.04\columnwidth}} 
\newcolumntype{R}{>{\centering\arraybackslash} m{0.48\columnwidth}} 
\newcolumntype{S}{>{\centering\arraybackslash} m{0.32\columnwidth}} 

\usepackage{algorithm}
\usepackage{algorithmic}

\usepackage[colorlinks = true]{hyperref}

\newtheorem{assumption}{Assumption}

\def\A{{\bf A}}

\def\B{{\bf B}}
\def\bb{{\bf b}}
\def\C{{\bf C}}

\def\D{{\bf D}}

\def\e{{\bf e}}

\def\G{{\bf G}}
\def\H{{\bf H}}
\def\I{{\bf I}}
\def\j{{\bf j}}

\def\k{{\bf k}}

\def\M{{\bf M}}

\def\PP{{\bf P}}

\def\S{{\bf S}}
\def\s{{\bf s}}

\def\U{{\bf U}}

\def\V{{\bf V}}
\def\v{{\bf v}}
\def\W{{\bf W}}
\def\w{{\bf w}}
\def\X{{\bf X}}
\def\x{{\bf x}}

\def\0{{\bf 0}}
\def\1{{\bf 1}}

\def\AM{{ A}}

\def\FM{{\mathscr F}}

\def\JM{{\mathscr J}}

\def\NM{{\mathcal N}}
\def\OM{{\mathcal O}}

\def\SM{{\mathscr S}}

\def\GB{{\mathbb G}}
\def\EB{{\mathbb E}}
\def\RB{{\mathbb R}}
\def\PB{{\mathbb P}}

\def\si{\mbox{\boldmath$\sigma$\unboldmath}}

\def\Xii{\mbox{\boldmath$\Xi$\unboldmath}}

\def\argmin{\mathop{\rm argmin}}

\def\card{\mathsf{card}}

\def\var{\mathsf{var}}

\def\cov{\mathrm{cov}}

\def\tr{\mathrm{tr}}
\def\rk{\mathrm{rank}}
\def\diag{\mathsf{diag}}

\def\ve{\varepsilon}

\newcommand{\ttop}{^{\top}}
\newcommand{\tsum}{\textstyle\sum}
\newcommand{\ts}{\textstyle}

\newtheorem{alg}{Algorithm}

\newcommand{\comment}[1]{}

\usepackage{lastpage}
\jmlrheading{20}{2019}{1-\pageref{LastPage}}{8/17; Revised 1/19}{2/19}{17-451}{Miles E. Lopes, Shusen Wang, and Michael W. Mahoney}
\ShortHeadings{A Bootstrap Method for Error Estimation in Randomized Matrix Multiplication}{Lopes, Wang, and Mahoney}

\begin{document}
\title{A Bootstrap Method for Error Estimation \\in Randomized Matrix Multiplication}

\author{\name Miles E.\ Lopes \email melopes@ucdavis.edu\\ 
	\addr Department of Statistics\\ 
	University of California at Davis\\ 
	Davis, CA 95616, USA
	\AND 
	\name Shusen Wang \email shusen.wang@stevens.edu\\ 
	\addr Department of Computer Science \\ 
	Stevens Institute of Technology\\ 
	Hoboken, NJ 07030, USA 
	\AND 
	\name Michael W.\ Mahoney \email mmahoney@stat.berkeley.edu \\ 
	\addr International Computer Science Institute and Department of Statistics\\ 
	University of California at Berkeley\\ 
	Berkeley, CA 94720, USA
}

\editor{Hui Zou}

\maketitle

\begin{abstract}%
In recent years, randomized methods for numerical linear algebra have received growing interest as a general approach to large-scale problems. Typically, the essential ingredient of these methods is some form of randomized dimension reduction, which accelerates computations, but also creates random approximation error. In this way, the dimension reduction step  encodes a tradeoff between cost and accuracy. However, the exact numerical relationship between cost and accuracy is typically unknown, and consequently, it may be difficult for the user to precisely know (1) how accurate a given solution is, or (2) how much computation is needed to achieve a given level of accuracy. In the current paper, we study randomized matrix multiplication (sketching) as a prototype setting for addressing these general problems. As a solution, we develop a bootstrap method for \emph{directly estimating} the accuracy as a function of the reduced dimension (as opposed to deriving worst-case bounds on the accuracy in terms of the reduced dimension). From a computational standpoint, the proposed method does not substantially increase the cost of standard sketching methods, and this is made possible by an ``extrapolation'' technique. In addition, we provide both theoretical and empirical results to demonstrate the effectiveness of the proposed method.
\end{abstract}

\begin{keywords}
matrix sketching, randomized matrix multiplication, bootstrap methods
\end{keywords}


\section{Introduction}
\label{sec:intro}

The development of randomized numerical linear algebra (RNLA or RandNLA) has led to a variety of efficient methods for solving large-scale matrix problems,  such as matrix multiplication, least-squares approximation, and low-rank matrix factorization, among others~\citep{halko2011random,mahoney2011randomized,woodruff2014sketching,drineas2016randnla}. 
 A general feature of these methods is that they apply some form of randomized dimension 
 reduction to an input matrix, which reduces the cost of subsequent computations. 
 In exchange for the reduced cost, the randomization leads to some error in the resulting solution, and consequently, there is a tradeoff between cost and accuracy.

For many canonical matrix problems, the relationship between cost and accuracy 
has been the focus of a growing body of theoretical work, and the literature provides 
many performance guarantees for RNLA  methods.
In general, these guarantees offer a good qualitative description of how 
the accuracy depends on factors such as problem size, number of iterations, 
condition numbers, and so on. 
Yet, it is also the case that such guarantees tend to be overly pessimistic for any particular problem instance --- often because the guarantees are formulated to hold in the worst case among a large class of possible inputs.
Likewise, it is often impractical to use such guarantees to determine precisely how accurate a given solution is, or precisely how much computation is needed to achieve a desired level of accuracy.

In light of this situation, it is of  interest to develop efficient methods for estimating the exact relationship between the cost and accuracy of RNLA methods on a \emph{problem-specific basis}. 
Since the literature has been somewhat quiet on this general question, the aim of this paper is to analyze randomized matrix multiplication as a prototype setting, and propose an approach that may be pursued more broadly. (Extensions are discussed at the end of the paper in Section~\ref{sec:extend}.)

\subsection{Randomized matrix multiplication}

To describe our problem setting, we briefly review the rudiments of randomized matrix multiplication, which is often known as matrix sketching~\citep{drineas06fastmonte1,mahoney2011randomized,woodruff2014sketching}. 
If $\A\in\mathbb{R}^{n\times d}$ and $\B\in\mathbb{R}^{n\times d'}$ are fixed input matrices, then sketching methods are commonly used to approximate $\A^T\B$ in the the regime where $\max\{d,d'\}\ll n$. For instance, this regime corresponds to ``big data'' applications where $\A$ and $\B$ are data matrices with very large numbers of observations. 
 
 As a way of reducing the cost of ordinary matrix multiplication, the main idea of sketching is to compute the product $\tilde{\A}^T\tilde{\B}$ of  smaller matrices $\tilde{\A}\in\RB^{t\times d}$ and $\tilde{\B}\in\RB^{t\times d'}$, for some choice of $t\ll n$. These smaller matrices are referred to as ``sketches'', and they are generated randomly according to
\begin{equation}
\small
\tilde{\A}: = \S\A \text{ \ \ \ and  \ \ \  } \tilde{\B}:=\S\B,
\end{equation}
where $\S\in\RB^{t\times n}$ is a random ``sketching matrix'' satisfying the condition
\begin{equation}\label{eqn:unbiased}
\EB[\S^T\S]=\I_n,
\end{equation}
with $\I_n$ being the identity matrix.
In particular, the relation~\eqref{eqn:unbiased} implies that the sketched product is an unbiased estimate,
$\EB[\tilde{\A}^T\tilde{\B}] = \A^T\B$.
Most commonly, the matrix $\S$ can be interpreted as acting on $\A$ and $\B$ by sampling their rows, or  by randomly projecting their columns. In Section~\ref{sec:pre}, we describe some popular examples of sketching matrices to be considered in our analysis.

\subsection{Problem formulation} 

When sketching is implemented, the choice of the sketch size $t$ plays a central role, since it directly controls the relationship between cost and accuracy. If $t$ is small, then the  sketched product $\tilde{\A}^T\tilde{\B}$ may be computed quickly, but it is unlikely to be a good approximation to $\A^T\B$. Conversely, if $t$ is large, then the sketched product is more expensive to compute, but it  is more likely to be accurate. For this reason, we will parameterize  the relationship between cost and accuracy in terms of $t$.

Conventionally, the error of an approximate matrix product is measured with a norm, and in particular, we will consider error as measured by the $\ell_{\infty}$-norm,
\begin{equation} \label{eq:alg:zt}
\varepsilon_t \; := \; \big\| \A^T \S^T \S \B - \A^T \B \big\|_\infty,
\end{equation}
where $\|\C\|_{\infty}:=\max_{i,j}|c_{ij}|$ for a matrix $\C=[c_{ij}]$. (Further background on analysis of $\ell_{\infty}$-norm or entry-wise error for matrix multiplication may be found in~\citep{higham2002,drineas06fastmonte1,demmel2007,pagh2013}, among others.)
In the context of sketching, it is crucial to note that $\ve_t$ is a random variable, due to the randomness in $\S$.
Consequently, it is natural to study the quantiles of $\varepsilon_t$, because they specify the \emph{tightest possible bounds} on $\varepsilon_t$ that hold with a prescribed probability.  
More specifically,  for any $\alpha\in(0,1)$, the $(1-\alpha)$-quantile of $\varepsilon_t$ is defined as
\begin{equation}\label{eqn:quantiledef}
\begin{split}
q_{1-\alpha}(t)
\; := \; \inf \big\{q \in[0,\infty) \: \big| \: 
\PB\big(   \ve_t \leq \, q\big) \,\geq\, 1-\alpha\big\}.
\end{split}
\end{equation}
For example, the quantity $q_{0.99} (t)$ is the tightest upper bound on $\ve_t$ that holds with probability at least 0.99.
Hence, for any fixed $\alpha$, the function $q_{1-\alpha}(t)$ represents a precise \emph{tradeoff curve} for relating cost and accuracy. Moreover, the function $q_{1-\alpha}(t)$ is specific to the input matrices $\A$ and $\B$. 

\begin{figure}[!t]
	\begin{center}
		\centering
		\includegraphics[width=0.49\textwidth]{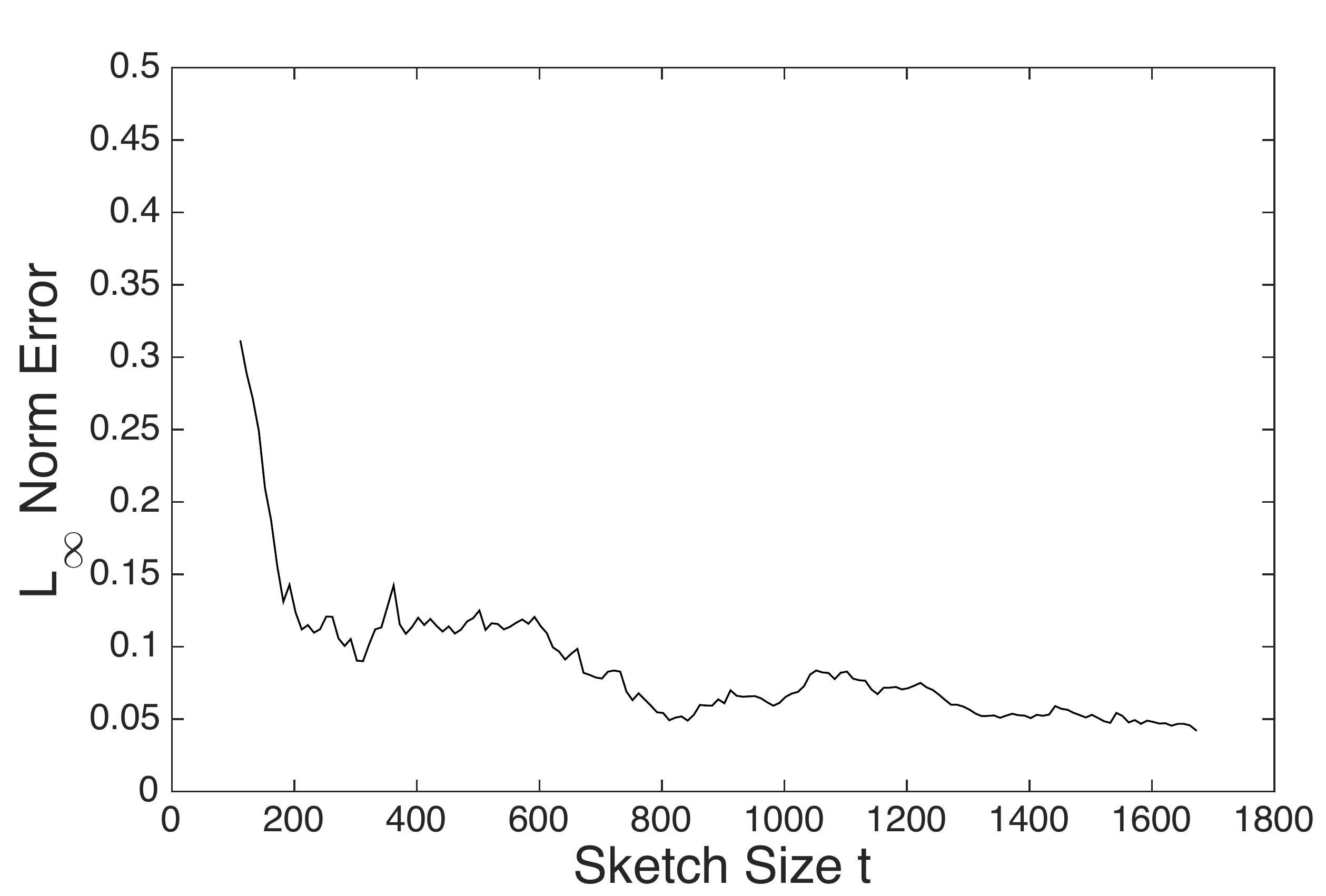}
		\includegraphics[width=0.49\textwidth]{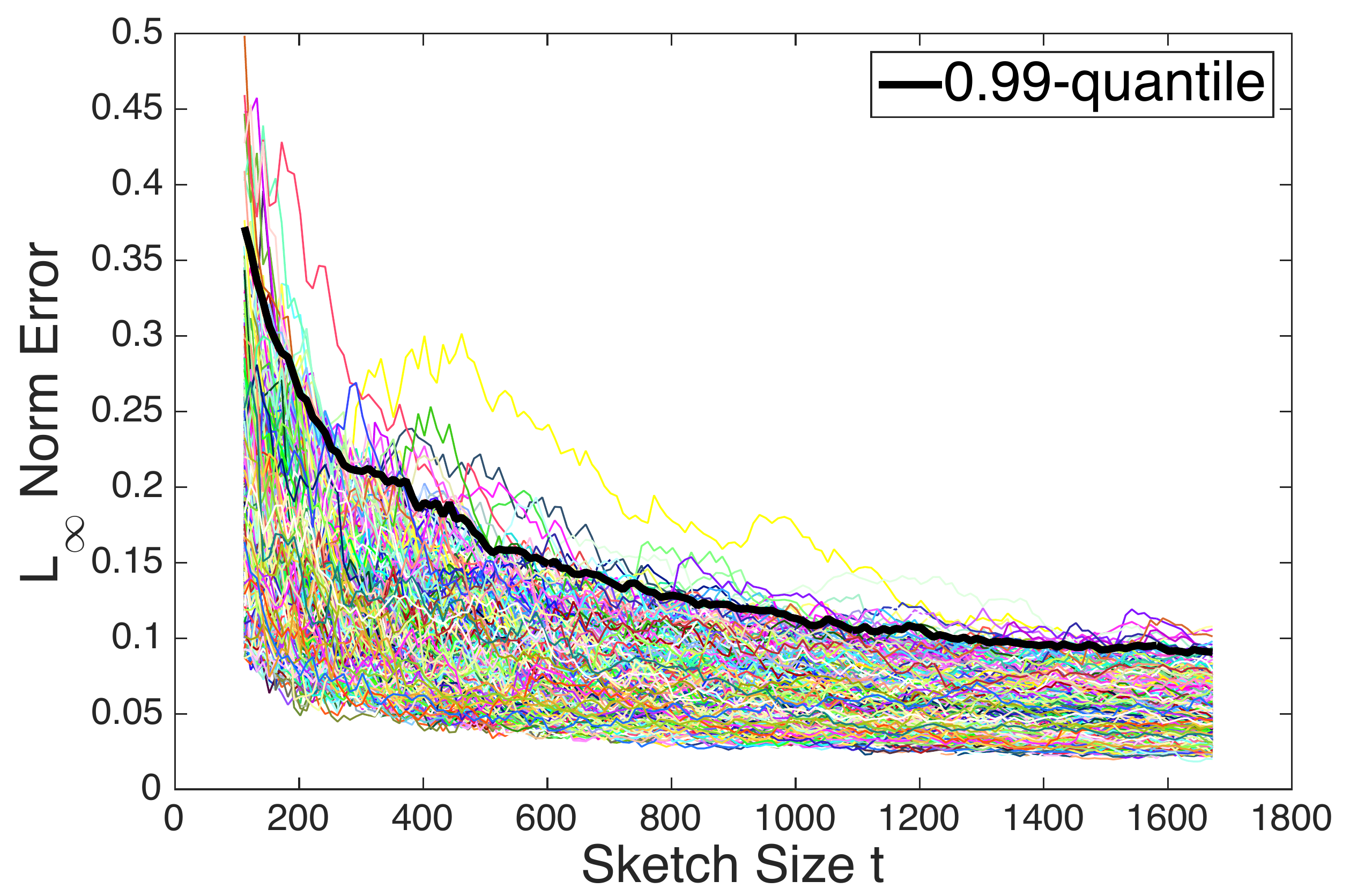}
	\end{center}
	\caption{{\textbf{Left panel:}} The curve shows how $\ve_t$ fluctuates with varying sketch size $t$, as rows are added to $\S$, with $\A$ and $\B$ held fixed. (Each row of $\A \in \RB^{8,124\times 112}$ is a feature vector of the Mushroom dataset \citep{uci2010}, and we set $\B=\A$.) The rows of $\S$ were generated randomly from a Gaussian distribution (see Section~\ref{sec:pre}), and the matrix $\A$ was scaled so that $\|\A^T \A\|_\infty = 1$. {\textbf{Right panel:}} There are 1,000 colored curves, each arising from a repetition of the simulation in the left panel. The thick black curve represents $q_{0.99}(t)$.}
	\label{fig:illustrate_zt}
\end{figure}
%

To clarify the interpretation of $q_{1-\alpha}(t)$, it is helpful to plot the fluctuations of $\ve_t$. In the left panel of Figure~\ref{fig:illustrate_zt}, we illustrate a simulation where randomly generated rows are incrementally added to a sketching matrix $\S$, with $\A$ and $\B$ held fixed. Each time a row is added to $\S$, the sketch size $t$ increases by 1, and we plot the corresponding value of $\ve_t$ as $t$ ranges from 100 to 1,700. (Note that the user is typically unable to observe such a curve in practice.) In the right panel, we display 1,000 repetitions of the simulation, with each colored curve corresponding to one repetition. (The variation is due only to the different draws of $\S$.)
In particular, the function $q_{0.99}(t)$ is represented by the thick black curve, delineating the top 1\% of the colored curves at each value of $t$.

In essence, the right panel of Figure~\ref{fig:illustrate_zt} shows that if the user had knowledge of the (unknown) function $q_{1-\alpha}(t)$, then two important purposes could be served. First, for any fixed value $t$, the user would have a sharp problem-specific bound on $\ve_t$. Second, for any fixed error tolerance $\epsilon$, the user could select $t$ so that that ``just enough'' computation is spent in order to achieve $\ve_t\leq \epsilon$ with probability at least $1-\alpha$.

\paragraph{The estimation problem.} The challenge we face is that a naive computation of $q_{1-\alpha}(t)$  by generating samples of $\ve_t$ would defeat the purpose of sketching. Indeed, generating samples of $\ve_t$ by brute force would require running the sketching method many times, and it would also require computing the entire product $\A^T\B$. Consequently, the technical problem of interest is to develop an efficient way to estimate $q_{1-\alpha}(t)$, without adding much cost to a \emph{single run} of the sketching method.

\subsection{Contributions} 

From a conceptual standpoint, the main novelty of our work is that it bridges two sets of ideas that are ordinarily studied in distinct communities. 
Namely, we apply the statistical technique of bootstrapping to enhance algorithms for numerical linear algebra.
To some extent, this pairing of ideas might seem counterintuitive, since bootstrap methods are sometimes labeled as ``computationally intensive'', but it will turn out that the cost of bootstrapping can be managed in our context.
Another reason our approach is novel is that we use the bootstrap to quantify error in the output of a randomized algorithm, rather than for the usual purpose of quantifying uncertainty arising from data.
In this way, our approach harnesses the versatility of bootstrap methods, and we hope that our results in the ``use case'' of matrix multiplication will encourage broader applications of bootstrap methods in randomized computations. (See also Section~\ref{sec:extend}, and note that in concurrent work, we have pursued similar approaches in the contexts of randomized least-squares and classification algorithms~\citep{LopesICML,Lopes2019}.)

From a technical standpoint, our main contributions are a method for estimating the function $q_{1-\alpha}(t)$, as well as theoretical performance guarantees.
Computationally, the proposed method is efficient in the sense that its cost is comparable to a single run of standard sketching methods (see Section~\ref{sec:pre}).
This efficiency is made possible by an ``extrapolation'' technique, which allows us to bootstrap small ``initial'' sketches with $t_0$ rows, and inexpensively estimate $q_{1-\alpha}(t)$ at larger values $t\gg t_0$. 
The empirical performance of the extrapolation technique is also quite encouraging, as discussed in Section~\ref{sec:exp}.
Lastly, with regard to theoretical analysis, our proofs circumvent some technical restrictions occurring in the analysis of related bootstrap methods in the statistics literature.

\subsection{Related work} 

Several works have considered the problem of error estimation for randomized matrix computations---mostly in the context of low-rank approximation~\citep{woolfe2008,liberty2007,halko2011random}, least squares~\citep{LopesICML}, or matrix multiplication~\citep{blum,sarlos2006}. 
With attention to matrix multiplication, the latter two papers offer methods for estimating high-probability bounds on the error $\eta_t:=\|\tilde{\A}^T\tilde{\B}-\A^T\B\|$, where $\|\cdot \|$ is either the maximum absolute row sum norm, or the Frobenius norm.
At a high level, all of the mentioned papers rely on a common technique, which is to randomly generate a sequence of ``test-vectors'', say $\v_1,\v_2,\dots$, and then use the matrix-vector products $\w_i:=\tilde{\A}^T\tilde{\B}\v_i-\A^T(\B\v_i)$ to derive an estimated bound, say $\hat{\eta}_t$, for $\eta_t$. The origin of this technique may be traced to the classical works~\citep{dixon1983,freivalds1979}.

Our approach differs from the ``test-vector approach'' in some essential ways. One difference arises because the bounds on $\hat{\eta}_t$ are generally constructed from the vectors $\{\w_i\}$ using conservative inequalities. By contrast, our approach avoids this conservativeness by \emph{directly estimating} $q_{1-\alpha}(t)$, which is an optimal bound on $\ve_t$ in the sense of equation~\eqref{eqn:quantiledef}. 

A second difference deals with computational demands. For example, in order to compute the vectors $\{\w_i\}$ in the test-vector approach, it is necessary to access the full matrices $\A$ and $\B$. On the other hand, our method does not encounter this difficulty, because it only requires access to the much smaller sketches $\tilde{\A}$ and $\tilde{\B}$. Also, in the test-vector approach, the cost to compute each vector $\w_i$  is proportional to the large dimension $n$, while the cost to compute  $\hat{q}_{1-\alpha}(t)$ with our method is \emph{independent} of $n$. Finally, the test-vector approach can only be used to check if the product $\tilde{\A}^T\tilde{\B}$ is accurate after it has been computed, whereas our approach can be used to  dynamically ``predict'' an appropriate sketch size $t$ from a small ``initial'' sketching matrix (see Section~\ref{sec:extrap}).

With regard to the statistics literature, our work builds upon a line of research dealing with ``multiplier bootstrap methods'' in  high-dimensional problems~\citep{cckMultiplier, cckSuprema,cckCLT}. Such methods are well-suited to approximating the distributions of statistics such as
$\|\bar{\x}\|_{\infty}$, where $\bar{\x}\in\mathbb{R}^p$ denotes the sample average of $n$ independent mean-zero vectors, with $n\ll p$. More recently, this approach has been substantially extended to other ``max type'' statistics arising from sample covariance matrices~\citep{chang_biometrics,chen_ustat}. 
Nevertheless, the strong results in these works do not readily translate to our context, either because the statistics are substantially different from the $\ell_{\infty}$-norm~\citep{chang_biometrics}, or because of technical assumptions~\citep{chen_ustat}. For instance, if the results in the latter work are applied to a sample covariance matrix of the form \smash{$\ts\frac{1}{n}\sum_{i=1}^n (\x_i-\bar{\x})(\x_i-\bar{\x})\ttop$,} where $\x_1,\dots,\x_n\in\mathbb{R}^p$ are mean-zero i.i.d.~vectors, with \smash{$\x_1=(X_{11},\dots,X_{1p})$,} then it is necessary to make assumptions such as \smash{$\min_{j,k} \var(X_{1j}X_{1k})\geq c$,}  for some constant $c>0$.
 As this relates to the sketching context, note that the sketched product may be written as \smash{$\tilde{\A}^T\tilde{\B}=\ts\frac{1}{t}\sum_{i=1}^t \A^T\s_i\s_i^T\B$,} where $\s_1,\dots,\s_n\in\mathbb{R}^t$ are the rows of $\sqrt{t}\S$. It follows that analogous variance assumptions would lead to conditions on the matrices $\A$ and $\B$ that could be 
violated if any column of $\A$ or $\B$ has many small entries, or is sparse. By contrast, our results do not rely on such variance assumptions, and we allow the matrices $\A$ and $\B$ to be unrestricted. 

At a more technical level, the ability to avoid restrictions on $\A$ and $\B$ comes from our use of the L\'evy-Prohorov metric for distributional approximations --- which differs from the Kolmogorov metric that has been predominantly used in previous works on multiplier bootstrap methods. More specifically, analyses based on the Kolmogorov metric typically rely on ``anti-concentration inequalities''~\citep{cckMultiplier,cckAnti}, which ultimately lead to the mentioned variance assumptions. On the other hand, our approach based on the L\'evy-Prohorov metric does not require the use of anti-concentration inequalities. Finally it should be mentioned that the techniques used to control the LP metric are related to those that have been developed for bootstrap approximations via coupling inequalities as in~\cite{cckIncreasing}.

\paragraph{Outline.}
This paper is organized as follows.
Section~\ref{sec:pre}  introduces some technical background.
Section~\ref{sec:alg} describes the proposed bootstrap algorithm.
Section~\ref{sec:theory} establishes the main theoretical results, and then numerical performance is illustrated in Section~\ref{sec:exp}. Lastly, conclusions and extensions of the method are presented in Section~\ref{sec:extend}, and all
proofs are given in the appendices.

\section{Preliminaries} \label{sec:pre}

\paragraph{Notation and terminology.} 
The set $\{1, \dots , n\}$ is denoted as $[n]$. The $i$th standard basis vector is denoted as $\e_i$.
If $\C=[c_{ij}]$ is a real matrix, then $\|\C\|_F = (\sum_{i,j} c_{ij}^2)^{1/2}$ is the Frobenius norm,
and $\|\C\|_2$ is the spectral norm (maximum singular value). If $X$ is a random variable and $p\geq1$, we write $\|X\|_p=(\EB[|X|^p])^{1/p}$ for the usual $L_p$ norm.
If $\psi:[0,\infty)\to[0,\infty)$ is a non-decreasing convex function with $\psi(0)=0$, then the $\psi$-Orlicz norm of $X$ is defined  as
\smash{$\|X\|_{\psi}:=\inf\{r>0 \ |\: \ \EB [\psi(|X|/r )]\leq 1\}$.}
In particular, we define $\psi_p(x):= \exp (x^p) -1$ for $p\geq 1$, and we say that $X$ is sub-Gaussian when $\|X\|_{\psi_2}<\infty$, or sub-exponential when $\|X\|_{\psi_1}<\infty$.
In Appendix~\ref{app:technical}, Lemma~\ref{lem:orlicz} summarizes the facts about Orlicz norms that will be used.

We will use $c$ to denote a positive absolute constant that may change from line to line. The matrices $\A$, $\B$, and $\S$ are viewed as lying in a sequence of matrices indexed by the tuple $(d,d',t,n)$.
For a pair of generic functions $f$ and $g$, we write $f(d,d',t,n)\lesssim g(d,d', t,n)$ when there is a positive absolute constant $c$ so that $f(d,d', t,n)\leq c\, g(d,d',t,n)$ holds for all large values of $d,d',t,$ and $n$. Furthermore, if $a$ and $b$ are two quantities that satisfy both $a\lesssim b$ and $b\lesssim a$, then we write $a\asymp b$. Lastly, we do not use the symbols $\lesssim$ or $\asymp$ when relating random variables.



\paragraph{Examples of sketching matrices.}
Our theoretical results will deal with three common types of sketching matrices, reviewed below.
\begin{itemize}
\item \emph{Row sampling}. If $(p_1,\dots,p_n)$ is a probability vector, then $\S\in\RB^{t\times n}$ can be constructed by sampling its rows i.i.d.~from the set $\{\ts\frac{1}{\sqrt{t p_1}}\e_1,\dots,\ts\frac{1}{\sqrt{t p_n}}\e_n\}\subset\RB^n$, where the vector $\ts\frac{1}{\sqrt{tp_i}}\e_i$ is selected with probability $p_i$. Some of the most well known choices for the sampling probabilities include \emph{uniform sampling}, with $p_i\equiv1/n$, \emph{length sampling} \citep{drineas06fastmonte1,magen2011low}, with
\begin{equation}\label{eqn:lengthsampling}
 p_i=\frac{\|\e_i^T\A\|_2 \|\e_i^T\B\|_2}{\sum_{j=1}^n \|\e_j^T\A\|_2\|\e_j^T\B\|_2},
\end{equation}
and \emph{leverage score sampling}, for which further background may be found in the papers \citep{drineas2006sampling,drineas2008cur,DMMW12_JMLR}.

\item \emph{Sub-Gaussian projection}. 
Gaussian projection is the most well-known random projection method, and is sometimes referred to as the Johnson-Lindenstrauss (JL) transform \citep{johnson1984extensions}.
In detail, if $\G \in \RB^{t\times n}$ is a standard Gaussian matrix, with entries that are i.i.d.~samples from $\NM (0, 1)$,
then $\S = \frac{1}{\sqrt{t}} \G$ is a Gaussian projection matrix. More generally, the entries of $\G$ can be drawn i.i.d.~from a zero-mean sub-Gaussian distribution, which often leads to similar performance characteristics in RNLA applications.

\item \emph{Subsampled randomized Hadamard transform (SRHT)}. %
Let $n$ be a power of $2$, and define the Walsh-Hadamard matrix $\H_n$ recursively\footnote{The restriction that $n$ is a power of 2 can be relaxed with variants of SRHT matrices~\citep{blendenpik,boutsidis2013}.}
$$\H_n:=\left(\begin{array}{cc} \H_{n/2} & \H_{n/2}\\ \H_{n/2} & -\H_{n/2}\end{array}\right) \text{ \ \ with \ \  } \H_2:=\left(\begin{array}{cc} 1 & 1\\ 1 & -1\end{array}\right).$$
Next,
let $\D_n^{\circ} \in \RB^{n\times n}$ be random diagonal matrix with independent $\pm 1$ Rademacher variables along the diagonal, 
and let $\PP \in \RB^{t\times n}$ have rows uniformly sampled from $\{\ts\frac{1}{\sqrt{t/n}}\e_1,\dots,\ts\frac{1}{\sqrt{t/n}}\e_n\}$. 
Then, the $t\times n$ matrix
\begin{equation}\label{eqn:srhtdef}
\S = \PP (\ts\frac{1}{\sqrt{n}}\H_n) \D_n^{\circ}
\end{equation}
is called an SRHT matrix. This type of sketching matrix was introduced in the seminal paper~\citep{ailon2006}, and additional details regarding implementation may be found in the papers~\citep{drineas2011faster,wang2015practical}. (The factor $\ts\frac{1}{\sqrt{n}}$ is used so that $\ts\frac{1}{\sqrt{n}}\H_n$ is an orthogonal matrix.) 
An important property of SRHT matrices is that they can be multiplied with any $n\times d$ matrix in $\OM (n \cdot d\cdot \log t)$ time~\citep{ailon2009fast}, which is faster than the $\mathcal{O}(n\cdot d\cdot t)$ time usually required for a dense sketching matrix.
\end{itemize}

\vspace{2mm}
\section{Methodology} \label{sec:alg}

Before presenting our method in algorithmic form, we first explain the underlying intuition.

\vspace{1mm}
\subsection{Intuition for multiplier bootstrap method}
\label{sec:intuition}

If the row vectors of $\sqrt{t}\S$ are denoted  $\s_1, \dots , \s_t \in \RB^n$, then $\S^T\S$ may be conveniently expressed as a sample average
 \begin{equation}
 \S^T \S = \ts\frac{1}{t} \sum_{i=1}^t \s_i \s_i^T .
 \end{equation}
For row sampling, Gaussian projection, and SRHT, these row vectors satisfy $\EB [\s_i \s_i^T] = \I_n$. 
Consequently, if we define the random $d\times d'$ rank-1 (dyad) matrix
\begin{equation} \label{eq:theory:m}
\D_i \; = \; \A^T \s_i \s_i^T \B,
\end{equation}
then $\EB [\D_i ]= \A^T \B$, and it follows that the difference between the sketched and unsketched products can be viewed as a sample average of zero-mean random matrices
\begin{equation}\label{eqn:cltrep}
\A^T \S^T \S \B -\A^T\B\; = \; \ts\frac{1}{t} \sum_{i=1}^t (\D_i-\A^T\B).
\end{equation}
Furthermore, in the cases of length sampling and Gaussian projection, the matrices $\D_1,\dots,\D_t$ are independent, and in the case of SRHT sketches, these matrices are ``nearly'' independent. So, in light of the central limit theorem, it is natural to suspect that the random matrix~\eqref{eqn:cltrep} will be well-approximated (in distribution) by a matrix with Gaussian entries. In particular, if we examine the $(j_1,j_2)$ entry, then we may expect that $\e_{j_1}^T\big(\A^T \S^T \S \B -\A^T\B\big)\e_{j_2}$ will approximately follow the distribution $\NM(0,\ts\frac{1}{t}\sigma_{j_1,j_2}^2)$, where the unknown parameter $\sigma_{j_1,j_2}^2$ can be estimated with 
$$\hat{\sigma}^2_{j_1,j_2} \: := \: \ts\frac{1}{t} \sum_{i=1}^t \big(\e_{j_1}^T(\D_i-\A^T\S^T\S\B)\e_{j_2}\big)^2.$$
Based on these considerations, the idea of the proposed bootstrap method is to generate a random matrix whose $(j_1,j_2)$ entry is sampled from $\NM(0,\ts\frac{1}{t}\hat{\sigma}_{j_1,j_2}^2)$. It turns out that an efficient way of generating such a matrix is to sample i.i.d.~random variables \smash{$\xi_1,\dots,\xi_t\sim \NM(0,1)$,} independent of $\S$, and then compute
\begin{equation}\label{eqn:bootstrapmimic}
\ts\frac{1}{t}\sum_{i=1}^t\xi_i \big(  \D_i - \A^T \S^T \S \B \big).
\end{equation}
In other words, if $\S$ is conditioned upon, then the distribution of the $(j_1,j_2)$ entry of the above matrix is exactly $\NM(0,\ts\frac{1}{t}\hat{\sigma}_{j_1,j_2}^2)$.\footnote{It is also possible to show that the  \emph{joint} distribution of the entries in the matrix~\eqref{eqn:bootstrapmimic} mimics that of $\A^T\S^T\S\B-\A^T\B$, but we omit such details to simplify the discussion.}
Hence, if the matrix~\eqref{eqn:bootstrapmimic} is viewed as an ``approximate sample'' of $\A^T\S^T\S \B-\A^T\B$, then it is natural to use the $\ell_{\infty}$-norm of the matrix~\eqref{eqn:bootstrapmimic} as an approximate sample of $\ve_t=\|\A^T\S^T\S \B-\A^T\B\|_{\infty}$. 
Likewise, if we define the bootstrap sample
\begin{eqnarray} \label{eq:theory:zt2}
\ve_t^\star 
& := &  
\Big\| \ts\frac{1}{t}\sum_{i=1}^t\xi_i \Big(  \D_i - \A^T \S^T \S \B \Big) \Big\|_\infty,
\end{eqnarray}
then the bootstrap algorithm will generate i.i.d.~samples of $\ve_t^{\star}$, conditionally on $\S$. In turn, the $(1-\alpha)$-quantile of the bootstrap samples, say $\hat{q}_{1-\alpha}(t)$, can be used to estimate $q_{1-\alpha}(t)$.

\vspace{2mm}
\subsection{Multiplier bootstrap algorithm}

We now explain how proposed method can be implemented in just a few lines. This description also reveals the important fact that  the algorithm only requires access to the sketches $\tilde{\A}$ and $\tilde{\B}$ (rather than the full matrices $\A$ and $\B$).
Although the formula for generating samples of $\ve_t^{\star}$ given below may appear different from equation~\eqref{eq:theory:zt2}, it is straightforward to check that these are equivalent. Lastly, the choice of the number of bootstrap samples $B$ will be discussed at the end of subsection~\ref{sec:extrap}.

\newpage

\begin{mdframed}
	\begin{alg} \label{alg:bootstrap1}
		{\bf (Multiplier bootstrap for $\ve_t$).}\\[0.0cm]
		{\bf Input:} the number of bootstrap samples $B$, and the sketches $\tilde{\A}$ and $\tilde{\B}$.\\[0.0cm]
		{\bf For } $b=1,\dots,B$\; {\bf do} 
		\vspace{-0.2cm}
		\begin{enumerate}
			\item 
			Draw an i.i.d.~sample $\xi_1,\dots,\xi_t$ from $\NM (0,1)$, independent of $\S$;
			\item
			Compute the bootstrap sample $\ve^{\star}_{t,b}:=\big\|\bar{\xi}\cdot(\tilde{\A}^T \tilde{\B}) - \tilde{\A}^T \Xii \tilde{\B}  \big\|_{\infty}$, where $\bar{\xi}:=\ts\frac{1}{t}\sum_{i=1}^t \xi_i$ and $\Xii:=\diag(\xi_1,\dots,\xi_t)$.
		\end{enumerate}
		\vspace{-0.1cm}
		{\bf Return:} 
		$\hat{q}_{1-\alpha}(t) \longleftarrow$
		the $(1-\alpha)$-quantile of the values $\ve^{\star}_{t,1},\dots,\ve^{\star}_{t,B}$. 
	\end{alg}
\end{mdframed}

\normalsize
\subsection{Saving on computation with extrapolation}\label{sec:extrap}

In its basic form, the cost of Algorithm~\ref{alg:bootstrap1} is $\OM (B\cdot t\cdot d\cdot d')$, which has the favorable property of being independent of the large dimension $n$. Also, the computation of the samples $\ve^{\star}_{t,1},\dots,\ve_{t,B}^{\star}$ is embarrassingly parallel, with the cost of each sample being $\OM(t\cdot d\cdot d')$.  Moreover, due to the way that the quantile $q_{1-\alpha}(t)$ scales with $t$, it is possible to reduce the cost of Algorithm~\ref{alg:bootstrap1} even further --- via the technique of extrapolation (also called Richardson extrapolation)~\citep{sidi,brezinski}.

The essential idea of extrapolation is to carry out Algorithm~\ref{alg:bootstrap1} for a modest ``initial'' sketch size $t_0$,  and then use an initial estimate $\hat{q}_{1-\alpha}(t_0)$ to ``look ahead'' and predict a larger value $t$ for which $q_{1-\alpha}(t)$ is small enough to satisfy the user's desired level of accuracy. The immediate benefit of this approach is that Algorithm~\ref{alg:bootstrap1} only needs to applied to small ``initial versions'' of $\tilde{\A}$ and $\tilde{\B}$, each with $t_0$ rows, which reduces the cost of the algorithm to $\OM (B\cdot  t_0\cdot  d\cdot d')$. Furthermore, this means that if Algorithm~\ref{alg:bootstrap1} is run in parallel, then it is only necessary to communicate copies of the small initial sketching matrices. (To illustrate the small size of the initial sketching matrices, our experiments include several examples where the ratio $t_0/n$ is approximately 1/100 or less.)

From a theoretical viewpoint, our use of extrapolation is based on the approximation $q_{1-\alpha}(t)\approx \ts\frac{\kappa}{\sqrt{t}}$, where $t$ is sufficiently large, and $\kappa=\kappa(\A,\B,\alpha)$ is an unknown number. A formal justification for this approximation can be made using Proposition~\ref{prop:gaussian} in Appendix~\ref{app:approx}, but it is simpler to give an intuitive explanation here. 
Recall from Section~\ref{sec:intuition} that as $t$ becomes large, the $(j_1,j_2)$ entry $[\tilde{\A}^T\tilde{\B}-\A^T\B]_{j_1,j_2}$ should be well-approximated in distribution by a Gaussian random variable of the form $\ts\frac{1}{\sqrt{t}}G_{j_1,j_2}$. In turn, this suggests that $\ve_t$ should be well-approximated in distribution by $\ts\frac{1}{\sqrt{t}}\max_{j_1,j_2}|G_{j_1,j_2}|$, which has quantiles that are proportional to $\ts\frac{1}{\sqrt{t}}$.

In order to take advantage of the theoretical scaling $q_{1-\alpha}(t)\approx \ts\frac{\kappa}{\sqrt{t}}$, we may use Algorithm~\ref{alg:bootstrap1} to compute $\hat{q}_{1-\alpha}(t_0)$ with an initial sketch size $t_0$, and then approximate the value $q_{1-\alpha}(t)$ for  $t\gg t_0$ with the following extrapolated estimator
\begin{equation}\label{extrapest}
 \hat{q}_{1-\alpha}^{\text{ ext}}(t):=\ts\frac{\sqrt{t_0}}{\sqrt{t}} \hat{q}_{1-\alpha}(t_0).
\end{equation}
Hence, if the user would like to determine a sketch size $t$ so that $q_{1-\alpha}(t)\leq \epsilon$, for some tolerance $\epsilon$, then $t$ should be selected so that $\hat{q}_{1-\alpha}^{\text{ ext}}(t)\leq \epsilon$, which is equivalent to
\begin{equation}\label{extraprule}
t\geq \Big(\ts\frac{\sqrt{t_0}}{\epsilon}\, \hat{q}_{1-\alpha}^{}(t_0)\Big)^2.
\end{equation}
 In our experiments in Section~\ref{sec:exp}, we illustrate some examples where an accurate estimate of $q_{1-\alpha}(t)$ at $t=10,\!000$ can be obtained from the rule~\eqref{extraprule} using an initial sketch size $t_0\approx 500$, yielding a roughly 20-fold speedup on the basic version of Algorithm~\ref{alg:bootstrap1}.
 
 \paragraph{Comparison with the cost of sketching.} Given that the purpose of Algorithm~\ref{alg:bootstrap1} is to enhance sketching methods, it is important to understand how the added cost of the bootstrap compares to the cost of running sketching methods in the standard way. As a point of reference, we compare with the cost of computing $\A^T\S^T\S\B$ when $\S$ is chosen to be an SRHT matrix, since this is one of the most efficient sketching methods. If we temporarily assume for simplicity that $\A$ and $\B$ are both of size $n\times d$,  then it follows from Section~\ref{sec:pre} that computing $\A^T\S^T\S\B$ has a  cost of order  $\mathcal{O}(t\cdot d^2+ n\cdot d\cdot \log(t))$. Meanwhile, the cost of running Algorithm~\ref{alg:bootstrap1} with the extrapolation speedup based on an initial sketch size $t_0$ is $\mathcal{O}(B\cdot t_0\cdot d^2)$. Consequently, the extra cost of the bootstrap does not exceed the stated cost of sketching when the number of bootstrap samples satisfies
\begin{equation}\label{bootstrapcost}
B=\mathcal{O}(\ts\frac{t}{t_0}+\ts\frac{n\log(t)}{d\, t_0}),
\end{equation}
 and in fact, this could be improved further if parallelization of Algorithm~\ref{alg:bootstrap1} is taken into account.  It is also important to note that rather small values of $B$ are shown to work well in our experiments, such as $B=20$. Hence, as long $t_0$ remains fairly small compared to $t$, then the condition~\eqref{bootstrapcost} may be expected to hold, and this is borne out in our experiments. The same reasoning also applies when $n \log(t)\gg d\cdot t_0$, which conforms with the fact that sketching methods are intended to handle situations where $n$ is very large.

 \subsection{Relation with the non-parametric bootstrap}
 For readers who are more familiar with the ``non-parametric bootstrap'' (based on sampling with replacement), the purpose of this short subsection is to explain the relationship with the multiplier bootstrap in Algorithm~\ref{alg:bootstrap1}. Indeed, an understanding of this relationship may be helpful, since the non-parametric bootstrap might be viewed as more intuitive, and perhaps easier to generalize to more complex situations. However, it turns out that Algorithm~\ref{alg:bootstrap1} is technically more convenient to analyze, and that is why the paper focuses primarily on~Algorithm~\ref{alg:bootstrap1}. Meanwhile, from a practical point of view, there is little difference between the two approaches, since both have the same order of computational cost, and in our experience, we have observed essentially the same performance in simulations. Also, the extrapolation technique can be applied to both algorithms in the same way.
 
 To spell out the connection, the only place where Algorithm~\ref{alg:bootstrap1} needs to be changed is in step 1. Rather than choosing the multiplier variables $\xi_1,\dots,\xi_t$ to be i.i.d.~$\NM(0,1)$ as in Algorithm~\ref{alg:bootstrap1}, the non-parametric bootstrap chooses $\xi_i=\zeta_i-1$, where  $(\zeta_1,\dots,\zeta_t)$ is a sample from a multinomial distribution, based on tossing $t$ balls into $t$ equally likely bins, where $\zeta_i$ is the number of balls in bin $i$. Hence, the mean and variance of each $\xi_i$ are nearly the same as before, with $\mathbb{E}[\xi_i]=0$ and $\var(\xi_i)=1-1/t$, but the variables $\xi_1,\dots,\xi_t$ are no longer independent.
 
 From a more algorithmic viewpoint, it is simple to check that the choice of $\xi_1,\dots,\xi_t$ based on the multinomial distribution is equivalent to  sampling with replacement from the rows of $\tilde{\A}$ and $\tilde{\B}$. The underlying intuition for this approach is based on the fact that for many types of sketching matrices, the rows of $\S$ are i.i.d., which makes  the rows of $\tilde{\A}$ i.i.d.,  and likewise for $\tilde{\B}$. Hence, if $\S$ is conditioned upon, then sampling with replacement from the rows of $\tilde{\A}$ and $\tilde{\B}$ imitates the random mechanism that originally generated $\tilde{\A}$ and $\tilde{\B}$.

\begin{mdframed}
	\begin{alg} \label{alg:bootstrap2}
		{\bf (Non-parametric bootstrap for $\ve_t$).}\\[0.2cm]
		{\bf Input}: the number of samples $B$, and the sketches $\tilde{\A}$ and $\tilde{\B}$.\\[0.2cm]
		{\bf For } $b=1,\dots,B$\; {\bf do} 
		\begin{enumerate}
			\item Draw a vector $(i_1,\dots,i_t)$ by sampling  $t$ numbers with replacement from $\{1,\dots,t\}$.
			\item Form matrices $\tilde{\A}^*\in\RB^{t\times d}$ and $\tilde{\B}^*\in\RB^{t\times d'}$ by  selecting (respectively) the rows from $\tilde{\A}$ and $\tilde{\B}$  that are indexed by $(i_1,\dots,i_t)$.
			\item
			Compute the bootstrap sample $\ve^{*}_{t,b}:=\big\|(\tilde{\A}^*)^{T} (\tilde{\B}^*) - \tilde{\A}^T \tilde{\B}  \big\|_{\infty}$.
		\end{enumerate}
		{\bf Return:} 
		$\hat{q}_{1-\alpha}(t) \longleftarrow$
		the $(1-\alpha)$-quantile of the values $\ve^{*}_{t,1},\dots,\ve^{*}_{t,B}$. 
	\end{alg}
\end{mdframed}
\normalsize


\section{Main results} \label{sec:theory}
Our main results quantify how well the estimate $\hat{q}_{1-\alpha}(t)$ from Algorithm 1 approximates the true value $q_{1-\alpha}(t)$, and this will be done by analyzing how well the distribution of a bootstrap sample $\ve_{t,1}^{\star}$ approximates the distribution of $\ve_t$.  For the purposes of comparing distributions, we will use the L\'evy-Prohorov metric, defined below.


\paragraph{L\'evy-Prohorov (LP) metric.} \label{sec:pre:metric}
Let $\mathcal{L}(U)$ denote the distribution of a random variable $U$, and
let $\mathscr{B}$ denote the collection of Borel subsets of $\RB$.  For any $\AM \in\mathscr{B}$, and $\delta>0$, define  the $\delta$-neighborhood
$\AM^{\delta}
\; := \; 
\big\{x\in\RB \: \big|\: \inf_{y\in A}|x-y|\leq \delta \big\}$.
Then, for any two random variables $U$ and $V$, the $d_{\text{LP}}$ metric between their distributions is given by
\begin{equation*}\label{prohorovdef}
d_{\text{LP}}(\mathcal{L}(U),\mathcal{L}(V))
\; := \;
\inf\Big\{\delta>0 \: \Big| \: \PB(U\in \AM) \leq \PB(V\in \AM^{\delta})+\delta \; \text{ for all } \;
\AM \in\mathscr{B}\Big\}.
\end{equation*}
The $d_{\text{LP}}$~metric is a standard tool for comparing distributions, due to the fact that convergence with respect to $d_{\text{LP}}$ is equivalent to convergence in distribution~\citep[Theorem 2.9]{huberbook}.

\paragraph{Approximating quantiles.} 
An important property of the $d_{\text{LP}}$~metric is that if two distributions are close in this metric, then their quantiles are close in the following sense. Recall that if $F_U$ is the distribution function of a random variable $U$, then the $(1-\alpha)$-quantile of $U$ is the same as the generalized inverse  \smash{$F_U^{-1}(1-\alpha):=\inf\{q \in [0,\infty)\, | \, F_U(q)\geq 1-\alpha\}$.} Next, suppose that two random variables $U$ and $V$ satisfy
\begin{equation*}
d_{\text{LP}} \big(\mathcal{L}(U), \, \mathcal{L}(V) \big)
\; \leq \; \epsilon,
\end{equation*}
for some $\epsilon\in(0,\alpha)$ with $\alpha\in (0,1/2)$.
Then, the quantiles of $U$ and $V$  are close in the sense that
\begin{equation}\label{eqn:quantileineqs}
\big|F_{U}^{-1}(1-\alpha)-F_V^{-1}(1-\alpha)\big| \ \leq \ \psi_{\alpha}(\epsilon),
\end{equation}
where the function $\psi_{\alpha}(\epsilon):=F_U^{-1}(1-\alpha+\epsilon)-F_U^{-1}(1-\alpha-\epsilon)+\epsilon$ is strictly monotone, and satisfies $\psi_{\alpha}(0)=0$.
(For a proof, see Lemma~\ref{LPquantile} of Appendix~\ref{app:technical}.)
In light of this fact, it will be more convenient to express our results for approximating $q_{1-\alpha}(t)$ in terms of the $d_{\text{LP}}$ metric.

\subsection{Statements of results} \label{sec:theory:theorem}

Our main assumption involves three separate cases, corresponding to different choices of the sketching matrix $\S$.

\begin{assumption}\label{assumptions} The dimensions $d$ and $d'$ satisfy $d\asymp d'$. Also, there is a positive absolute constant $\kappa\geq 1$ such that $d^{1/\kappa}\lesssim t\lesssim  d^{\kappa}$, which is to say that neither $d$ nor $t$ grows exponentially with the other. In addition, one of the following sets of conditions holds, involving the parameter $\nu(\A,\B):=\sqrt{\|\A^T\A\|_{\infty}\|\B^T\B\|_{\infty}}$.\\[-.5cm]

\begin{enumerate}[(a)]
\item \noindent  (Sub-Gaussian case).
The entries of the matrix $\S=[S_{i,j}]$ are zero-mean i.i.d.~sub-Gaussian random variables, with $\EB[S_{i,j}^2]=\frac{1}{t}$, and $ \max_{i,j}\|\sqrt{t}S_{i,j}\|_{\psi_2}\lesssim 1$. Furthermore,  $t \; \gtrsim \; \nu(\A,\B)^{2/3} (\log d)^5 $. \\

\item \noindent (Length sampling case).
The matrix $\S$ is generated by length sampling, with the probabilities in equation~\eqref{eqn:lengthsampling}, and also, $t\; \gtrsim \; (\|\A\|_F\|\B\|_F)^{2/3} (\log d)^5$.\\

\item \noindent (SRHT case). The matrix $\S$ is an SRHT matrix as defined in equation~\eqref{eqn:srhtdef}, and also, $t\gtrsim \nu(\A,\B)^{2/3} (\log n)^2 (\log d)^5$.\\
\end{enumerate}
\end{assumption}

\paragraph{Clarifications on bootstrap approximation.} Before stating our main results below, it is worth clarifying a few technical items. First, since our analysis involves central limit type approximations of $\tilde{\A}^T\tilde{\B}-\A^T\B$ as a sum of $t$ independent matrices, we will rescale the error variables by a factor of $\sqrt{t}$, obtaining
\begin{equation}\label{ztdef}
Z _t:=\sqrt{t}\ve_t,
\end{equation}
as well as its bootstrap analogue,
\begin{equation}\label{ztstardef}
Z_t^{\star}:=\sqrt{t}\ve_t^{\star}.
\end{equation}
With regard to the original problem of estimating the quantile  $q_{1-\alpha}(t)$ for $\ve_t$, this rescaling makes no essential difference, since quantiles are homogenous with respect to scaling, and in particular, the $(1-\alpha)$-quantile of $Z_t$ is simply $\sqrt{t}q_{1-\alpha}(t)$. 

As a second clarification, recall that the bootstrap method generates samples $\ve_t^{\star}$ based upon a particular realization of $\S$.
For this reason, the bootstrap approximation to $\mathcal{L}(Z_t)$ is the conditional distribution $\mathcal{L}(Z_t^{\star}|\S)$. Consequently, it should be noted that $\mathcal{L}(Z_t^{\star}|\S)$ is a random probability measure, and $d_{\text{LP}}(\mathcal{L}(Z_t)\,,\: \mathcal{L}(Z_t^{\star}|\S))$ is a random variable, since they both depend on the random matrix $\S$.

\begin{theorem} \label{thm:main}
Let $h(x)=x^{1/2}+x^{3/4}$ for $x\geq 0$. If Assumption~\ref{assumptions} (a) holds, then there is an absolute constant $c>0$ such that the following bound holds with probability at least $1-\ts\frac{1}{t}-\ts\frac{1}{dd'}$, 
	\begin{eqnarray*}
		d_{\text{\emph{LP}}}
	\Big(\mathcal{L}(Z_t)\,,\: \mathcal{L}(Z_t^{\star}|\S)\Big)
	& \leq &  \frac{c \cdot  h(\nu(\A,\B))\cdot \sqrt{\log(d)} }{t^{1/8}}.
	\end{eqnarray*}
If Assumption~\ref{assumptions} (b) holds, then there is an absolute constant $c>0$ such that the following bound holds with probability at least $1-\ts\frac{1}{t}-\ts\frac{1}{dd'}$, 
		\begin{eqnarray*}
		d_{\text{\emph{LP}}}
	\Big(\mathcal{L}(Z_t)\,,\: \mathcal{L}(Z_t^{\star}|\S)\Big)
	& \leq &  \frac{c \cdot  h(\|\A\|_F\|\B\|_F)\cdot \sqrt{\log(d)} }{t^{1/8}}.
	\end{eqnarray*}
\end{theorem}

\paragraph{Remarks.}
A noteworthy property of the bounds is that they are \emph{dimension-free} with respect to the large dimension $n$. Also, they have a very mild logarithmic dependence on $d$. With regard to the dependence on $t$, there are two other important factors to keep in mind. First, the practical performance of the bootstrap method (shown in Section~\ref{sec:exp}) is much better than what the $t^{-1/8}$ rate suggests. Second, the problem of finding the optimal rates of approximation for multiplier bootstrap methods is a largely open problem --- even in the simpler setting of bootstrapping the coordinate-wise maximum of vectors (rather than matrices). In the vector context, the literature has focused primarily on the Kolmogorov metric (rather than the LP metric), and some quite recent improvements beyond the $t^{-1/8}$ rate have been developed in~\cite{cckCLT} and~\cite{Lopes:2018}. However, these works also rely on model assumptions that would lead to additional restrictions on the matrices $\mathbf{A}$ and $\mathbf{B}$ in our setup. Likewise, the problem of extending our results to achieve faster rates or handle other metrics is a natural direction for future work.

\paragraph{The SRHT case.}
For the case of SRHT matrices, the analogue of Theorem~\ref{thm:main} needs to be stated in a slightly different way for technical reasons. From a qualitative standpoint, the results for SRHT and sub-Gaussian  matrices turn out to be similar.

The technical issue to be handled is that the rows of an SRHT matrix are not independent, due to their common dependence on the matrix $\D_n^{\circ}$. Fortunately, this inconvenience can be addressed by conditioning on $\D_n^{\circ}$. Theoretically, this simplifies the analysis of the bootstrap, since it ``decouples'' the rows of the SRHT matrix. Meanwhile, if we let $\tilde{q}_{1-\alpha}(t)$ denote the $(1-\alpha)$-quantile of the  distribution $\mathcal{L}(\ve_t|\D_n^{\circ})$, 
$$\tilde{q}_{1-\alpha}(t) \ := \ \inf\Big\{q\in[0,\infty)\Big| \ \PB(\ve_t \leq q\big|\D_n^{\circ}) \geq 1-\alpha \Big\},$$
 then it is simple  to check that $\tilde{q}_{1-\alpha}(t)$ acts as a ``surrogate'' for $q_{1-\alpha}(t)$, since\footnote{It is also possible to show that $\tilde{q}_{1-\alpha}(t)$ fluctuates around $q_{1-\alpha}(t)$. Indeed, if we define the random variable $V:=\mathbb{P}(\ve_t\leq q_{1-\alpha}(t)|\D_n^{\circ})$, it can be checked that the event
$V\geq 1-\alpha$
is equivalent to the event
$\tilde{q}_{1-\alpha}(t) \leq q_{1-\alpha}(t).$
Furthermore, if we suppose that $1-\alpha$ lies in the range of the c.d.f.~of $\ve_t$, then $\mathbb{E}[V]=1-\alpha$. In turn, it follows that the event $\tilde{q}_{1-\alpha}(t)\leq q_{1-\alpha}(t)$ occurs when $V\geq \mathbb{E}[V]$, and conversely, the event $\tilde{q}_{1-\alpha}(t)> q_{1-\alpha}(t)$ occurs when $V< \mathbb{E}[V]$.} 

\begin{equation}
\begin{split}
\PB(\ve_t\leq \tilde{q}_{1-\alpha}(t)) &= \EB\big[ \PB\big(\ve_t\leq \tilde{q}_{1-\alpha}(t)\big|\D_n^{\circ}\big)\big]\\[0.2cm]
& \geq \EB[1-\alpha]\\[0.2cm]
&=1-\alpha.
\end{split}
\end{equation}
For this reason,  we will view $\tilde{q}_{1-\alpha}(t)$ as the new parameter to estimate (instead of $q_{1-\alpha}(t)$), and accordingly, the aim of the following result is to quantify how well the bootstrap distribution $\mathcal{L}(Z_t^{\star}|\S)$ approximates the conditional distribution $\mathcal{L}(Z_t|\D_n^{\circ})$.

\begin{theorem}\label{thm:main2} Let $h(x)=x^{1/2}+x^{3/4}$ for $x\geq 0$. If Assumption~\ref{assumptions} (c) holds, then there is an absolute constant $c>0$ such that the following bound holds with probability at least $1-\ts\frac{1}{t}-\ts\frac{1}{dd'}-\ts\frac{c}{n}$, 
		\begin{eqnarray*}
		d_{\text{\emph{LP}}}
	\Big(\mathcal{L}(Z_t|\D_n^{\circ})\,,\: \mathcal{L}(Z_t^{\star}|\S)\Big)
	& \leq &  \frac{c \cdot  h(\nu(\A,\B)\log(n))\cdot \sqrt{\log(d)} }{t^{1/8}}.
	\end{eqnarray*}
\end{theorem}

\paragraph{Remarks.}Up to a factor involving $\log(n)$, the bound for SRHT matrices matches that for sub-Gaussian matrices. Meanwhile, from a more practical standpoint, our empirical results will show that the bootstrap's performance for SRHT matrices is generally similar to that for both sub-Gaussian and length-sampling matrices.

\paragraph{Further discussion of results.} To comment on the role of $\nu(\A,\B)$ and $\|\A\|_F\|\B\|_F$ in Theorems~\ref{thm:main} and~\ref{thm:main2}, it is possible to interpret them as problem-specific ``scale parameters''. Indeed, it is natural that the bounds on $d_{\text{LP}}$ should increase with the scale of $\A$ and $\B$ for the following reason. Namely, if $\A$ or $\B$ is multiplied by a scale factor $\kappa>0$, then it can be checked that the quantile error $|\hat{q}_{1-\alpha}(t)-q_{1-\alpha}(t)|$ will also change by a factor of $\kappa$, and furthermore, the inequality~\eqref{eqn:quantileineqs} demonstrates a monotone relationship between the sizes of the quantile error and the $d_{\text{LP}}$  error.  For this reason, the bootstrap may still perform well in relation to the scale of the problem when the magnitudes of the parameters $\nu(\A,\B)$ and $\|\A\|_F\|\B\|_F$ are large. Alternatively, this idea can be seen by noting that the  $d_{\text{LP}}$ bounds can be made arbitrarily small by simply changing the units used to measure the entries of $\A$ and $\B$.

Beyond these considerations, it is still of interest to compare the results for different sketching matrices once a particular scaling has been fixed. For concreteness, consider a scaling where the spectral norms of $\A$ and $\B$ satisfy $\|\mathbf{A}\|_{2}\asymp \|\mathbf{B}\|_{2}\asymp 1$. (As an example, if we view $\mathbf{A}\ttop \mathbf{A}$ as a sample covariance matrix, then the condition $\|\mathbf{A}\|_{2}\asymp 1$ simply means that the largest principal component score is of order 1.) Under this scaling, it is simple to check that $\nu(\mathbf{A},\mathbf{B})=\mathcal{O}(1)$, and $\|\mathbf{A}\|_F\|\mathbf{B}\|_F=\mathcal{O}(\sqrt{\srk(\mathbf{A})\srk(\mathbf{B})})$, where $\srk(\mathbf{A}):=\|\mathbf{A}\|_F^2/\|\mathbf{A}\|_2^2$ is the ``stable rank''.  In particular, note that if $\mathbf{A}$ and $\mathbf{B}$ are approximately low rank, as is common in applications, then $\srk(\mathbf{A})\ll d$, and \smash{$\srk(\mathbf{B})\ll d'$.} Accordingly, we may conclude that if the conditions of Theorems~\ref{thm:main} and~\ref{thm:main2} hold, then bootstrap consistency occurs under the following limits
\begin{align}
\sqrt{\log(d)}/t^{1/8}= o(1) & \ \ \text{ in the sub-Gaussian case},\\[0.3cm]
(\srk(\mathbf{A})\srk(\mathbf{B}))^{3/8}\sqrt{\log(d)}/t^{1/8}= o(1) & \ \ \text{ in the length-sampling case},\\[0.2cm]
\log(n)^{3/4}\sqrt{\log(d)}/t^{1/8}= o(1) & \ \ \text{ in the SRHT case},
\end{align}
where we have used the simplifying assumption that $d\asymp d'$.

\section{Experiments} \label{sec:exp}

This section outlines a set of experiments for evaluating the performance of Algorithm~\ref{alg:bootstrap1} with the extrapolation speed-up described in Section~\ref{sec:extrap}.
The experiments involved both synthetic and natural matrices, as described below.


\paragraph{Synthetic matrices.}

In order to generate the matrix $\A\in\RB^{n\times d}$ synthetically,  we selected the factors of its singular value decomposition  $\A = \U \diag (\si ) \V^T$  in the following ways, fixing $n=30,000$ and $d=1,000$. In previous work, a number of other experiments in randomized matrix computations have been designed along these lines~\citep{ma2014statistical,yang2015implementing}.

The factor $\U\in\RB^{n\times d}$ was selected as the Q factor from the reduced QR factorization of a random matrix $\X \in \RB^{n\times d}$. The rows  of $\X$ were sampled i.i.d.~from a multivariate $t$-distribution, $\boldsymbol t_{2}(\boldsymbol \mu,\C)$, with $2$ degrees of freedom, mean $\boldsymbol \mu=\0$, and covariance  $c_{ij}=2\times 0.5^{|i-j|}$ where $\C=[c_{ij}]$. (This choice causes the matrix $\A$ to have high row-coherence, which is of interest, since this is a challenging case for sampling-based sketching matrices.) Next, the factor $\V\in\RB^{d\times d}$ was selected as the Q factor from a QR factorization of a $d\times d$ matrix with i.i.d.~$\NM(0,1)$ entries. For the singular values $\boldsymbol \sigma\in\RB_+^{d}$, we chose two options, leading to either a low or high stable rank $\srk(\A )= \tfrac{\|\A\|_F^2 }{ \| \A \|_2^2 }$. In the low stable rank case, we put  $\sigma_i=10^{\kappa_i}$ for a set of equally spaced values $\kappa_i$ between 0 and -6, yielding $\srk(\A)=36.7$. Alternatively, in the high stable rank case, the entries of $\boldsymbol \sigma$ were equally spaced between $0.1$ and $1$, yielding $\srk(\A)=370.1$. Finally, to make all numerical comparisons on a common scale, we normalized $\A$ so that $\|\A^T\A\|_{\infty}=1$.

\comment{
\begin{itemize}
	\item 
	{\bf Column bases.}
	Let the rows of $\X \in \RB^{n\times d}$ be i.i.d.\ sampled from multivariate $t$-distribution
	with covariance matrix $\C$ and $v=2$ degree of freedom, 
	where the $(i,j)$-th entry of $\C \in \RB^{m\times n}$ is $2\times 0.5^{|i-j|}$.
	Let $\U$ be the orthonormal bases of $\X$.
	In this way, $\U$ and $\A$ have high row coherence.
	\vspace{-1mm}
	\item
	{\bf Row bases.}
	Let $\V \in \RB^{d\times d}$ be the orthonormal bases of a $d\times d$ standard Gaussian matrix.
\end{itemize}
We construct $\si \in \RB_+^d$ in two different ways to get very different stable rank
$ (\A ) \triangleq \tfrac{\|\A\|_F^2 }{ \| \A \|_2^2 }$,
which is smaller than or equal to $\rk (\A )$.
\begin{itemize}
	\item
	{\bf Low stable rank.} Let the entries of $\bb\in \RB^{d}$ be equally paced between $0$ and $-6$.
	Let $\sigma_i = 10^{b_i}$ for all $i\in [d]$.
	For $d = 1,000$, the stable rank is
	$36.7$---far smaller than the rank of $\A$.
	\vspace{-1mm}
	\item
	{\bf High stable rank.}
	Let the entries of $\si$ be equally paced between $0.1$ and $1$.
	For $d = 1,000$, the stable ranks is 
	$370.1$ which is of the same order as the rank of $\A$.
\end{itemize}
We fix $n = 30,000$ and $d=1,000$.
}

\paragraph{Natural matrices.}
We also conducted experiments on five natural data matrices $\A$ from the LIBSVM repository~\citet{libsvm}, named `Connect', `DNA', `MNIST', `Mushrooms', and `Protein', with the same normalization that was used for the synthetic matrices.
These  datasets are briefly summarized in Table~\ref{tab:datasets}.

\begin{table}[h]\setlength{\tabcolsep}{0.3pt}
	\caption{A summary of the natural datasets.}
	\label{tab:datasets}
	\begin{center}
		\begin{small}
			\begin{tabular}{l c c c c c c  }
				\hline
				{\bf Dataset}  &~~Connect~~&~~~DNA~~~&~~MNIST~~&~~Mushrooms~~&~~Protein~~\\
				\hline
				{\ \ \ \ \ $n$} & $67,557$& $2,000$ & $60,000$&    $8,124$  & $17,766$ \\[0.05cm]
				{\ \ \ \ \ $d$} &   $126$ &  $180$  &   $780$ &     $112$   &  $356$  \\
				\hline
			\end{tabular}
		\end{small}
	\end{center}
\end{table}
\vspace{-.5cm}


\comment{

\begin{figure}[!ht]
	\begin{center}
		\centering
		\includegraphics[width=0.99\textwidth]{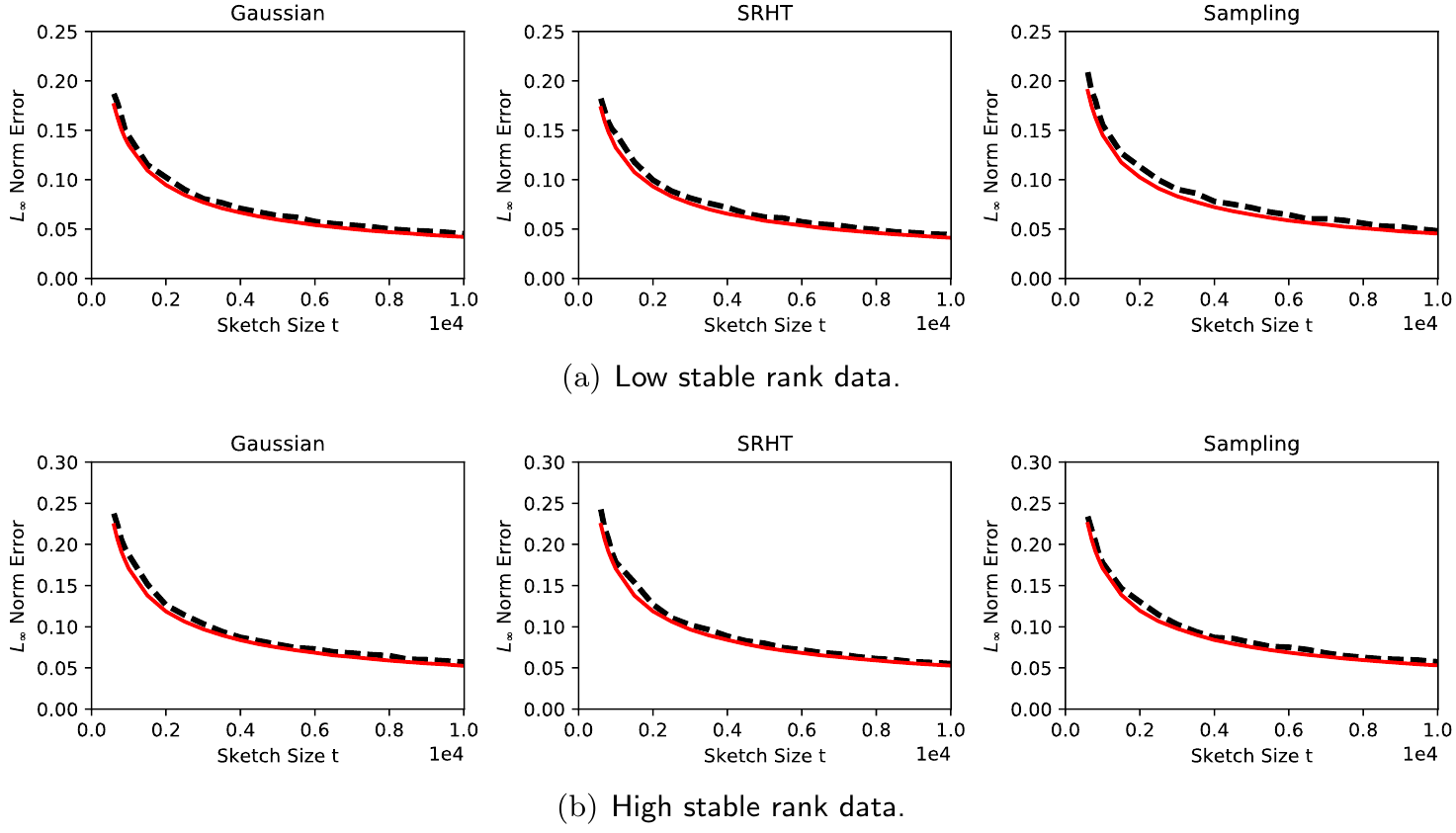}
	\end{center}
	\caption{Bootstrap experiments on the $n=30,000$ and $d=1,000$ synthetic data.
		The black dash is the empirical $0.99$ quantile of the $\ell_\infty$ norm error.
		The red line is the bootstrap estimation.}
	\label{fig:boot_d1k}
\end{figure}

}

\subsection{Design of experiments}\label{sec:exp:boot}
For each matrix $\A$, natural or synthetic, we considered the task of estimating the quantile $q_{0.99}(t)$ for the random sketching error $\ve_t=\big\| \A^T \A- \A^T \S^T \S \A \big\|_{\infty}$. The sketching matrix $\S\in\RB^{t\times n}$ was allowed to be one of three types: Gaussian projection, length-sampling, and SRHT, as described in Section~\ref{sec:pre}.

\paragraph{Ground truth values.}The ground truth values for $q_{0.99}(t)$ were constructed in the following way. For each matrix $\A$, a grid of $t$ values was specified, ranging from $d/2$ up to a larger number as high as $10d$ or $20d$, depending on $\A$. Next, for each $t$ value, and for each type of sketching matrix, we used 1,000 realizations of $\S\in\RB^{t\times n}$, yielding 1,000 realizations of the random variable $\ve_t$. In turn, the 0.99 sample quantile of the 1,000 realizations of $\ve_t$ was treated as the true value of $q_{0.99}(t)$, and this appears as the  black curve in all plots.

\paragraph{Extrapolated estimates.} With regard to the bootstrap extrapolation method in Section~\ref{sec:extrap}, we fixed the value $t_0=d/2$  as the initial sketch size to extrapolate from. For each $\A$, and each type of sketching matrix, we applied Algorithm~\ref{alg:bootstrap1} to each of the 1,000 realizations of $\tilde{\A}=\S\A\in \RB^{t_0\times d}$ generated previously. Each time Algorithm~\ref{alg:bootstrap1} was run, we used the modest choice of $B=20$ for the number of bootstrap samples. From each set of 20 bootstrap samples, we used the 0.99 sample quantile as the estimate $\hat{q}_{0.99}(t_0)$.\footnote{Note that since $19/20=0.95$ and $20/20=1$, the 0.99 quantile was obtained by an interpolation rule.} Hence, there were 1,000 realizations of $\hat{q}_{0.99}(t_0)$ altogether. Next, we used the scaling rule in equation~\eqref{extrapest} to obtain 1,000 realizations of the extrapolated estimate $\hat{q}_{0.99}^{ \, \text{ext} }(t)$ for values $t\geq t_0$.

In order to illustrate the variability of the estimate $\hat{q}_{0.99}^{ \, \text{ext} }(t)$ over the 1,000 realizations, we plot three different curves as a function of $t$. The  blue curve represents the average value of $\hat{q}_{0.99}^{ \, \text{ext} }(t)$, while the  green and yellow curves respectively correspond to the estimates ranking 100th an 900th out of the 1,000 realizations.


\subsection{Comments on numerical results}

Overall, the numerical results for the bootstrap extrapolation method are quite encouraging, and to a large extent, the method is accurate across many choices of $\A$ and $\S$. Given that the blue curves representing $\EB[\hat{q}_{0.99}^{\,\text{ext}}(t)]$ are closely aligned with the  black curves for $q_{0.99}(t)$, we see that the extrapolated estimate is essentially unbiased. Moreover, the variance of the estimate is fairly low, as indicated by the small gap between the green and yellow curves. The low variance is also notable when considered in light of the fact that only $B=20$ bootstrap samples are used to construct $\hat{q}^{\,\text{ext}}_{0.99}(t)$, since the variance should decrease as $B$ becomes larger.

With attention to the extrapolation rule~\eqref{extrapest}, there are two main points to note. First, the plots show that the extrapolation may be initiated at fairly low values of $t_0$, which are much less than the sketch sizes needed to achieve a small sketching error $\ve_t$.  Second, we see that $\hat{q}_{0.99}^{\, \text{ext}}(t)$ remains accurate for $t$ much larger than $t_0$, well up to $t=10,000$ and perhaps even farther.  Consequently, the results show that the extrapolation technique is capable of saving quite a bit of computation without much detriment to statistical performance. 

To consider the relationship between theory and practice, one basic observation is that all three types of sketching matrices obey roughly similar bounds in  Theorems~\ref{thm:main} and~\ref{thm:main2}, and indeed, we also see generally similar numerical performance among the three types. At a more fine-grained level however, the Gaussian and SRHT sketching matrices tend to produce estimates $\hat{q}^{\, \text{ext}}_{0.99}(t)$ with somewhat higher variance than in the case of length sampling. Another difference between theory and simulation, is that the actual performance of the method seems to be better than what the theory suggests --- since the estimates are accurate at values of $t_0$ that are much smaller than what would be expected from the rates in Theorems~\ref{thm:main} and~\ref{thm:main2}.

\begin{figure}[h]
	\begin{center}
		\centering
		\includegraphics[width=0.95\textwidth]{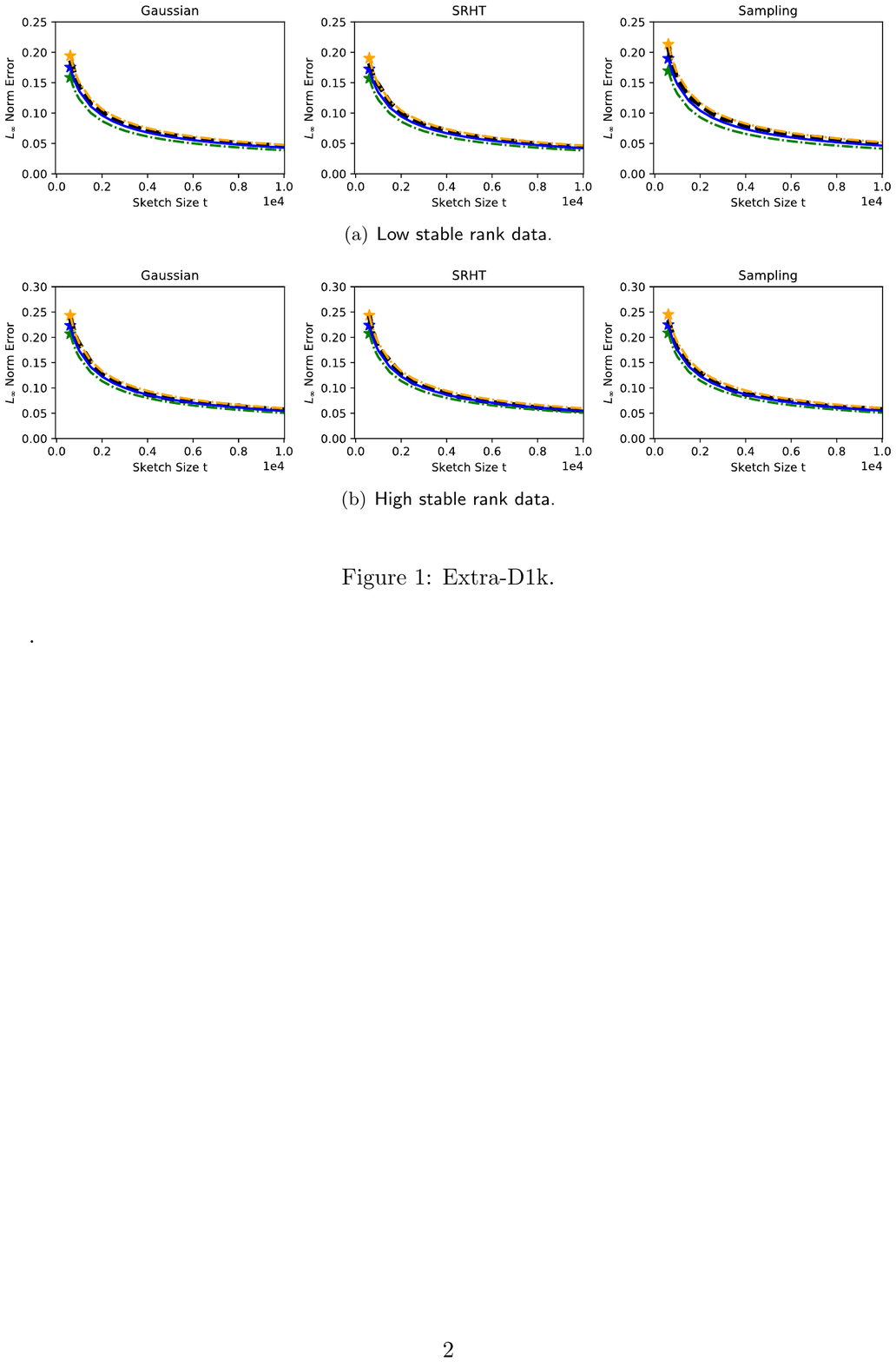}
	\end{center}
	\caption{{\textbf{Results for synthetic matrices.}} 
		The black line represents $q_{0.99}(t)$ as a function of $t$.
		The blue star is the average bootstrap estimate at the initial sketch size $t_0=d/2=500$, and the blue line represents the average extrapolated estimate $\EB[\hat{q}_{0.99}^{\,\text{ext}}(t)]$ derived from the starting value $t_0$. To display the variability of the estimates, the green and yellow curves correspond to the 100th and 900th largest among the 1,000 realizations of $\hat{q}_{0.99}^{\, \text{ext}}(t)$ at each $t$.}
	\label{fig:extra_d1k}
\vspace{-0.5cm}
\end{figure}

\begin{figure}[h!]
	\vspace{-10mm}
	\begin{center}
		\centering
		\includegraphics[width=0.9\textwidth]{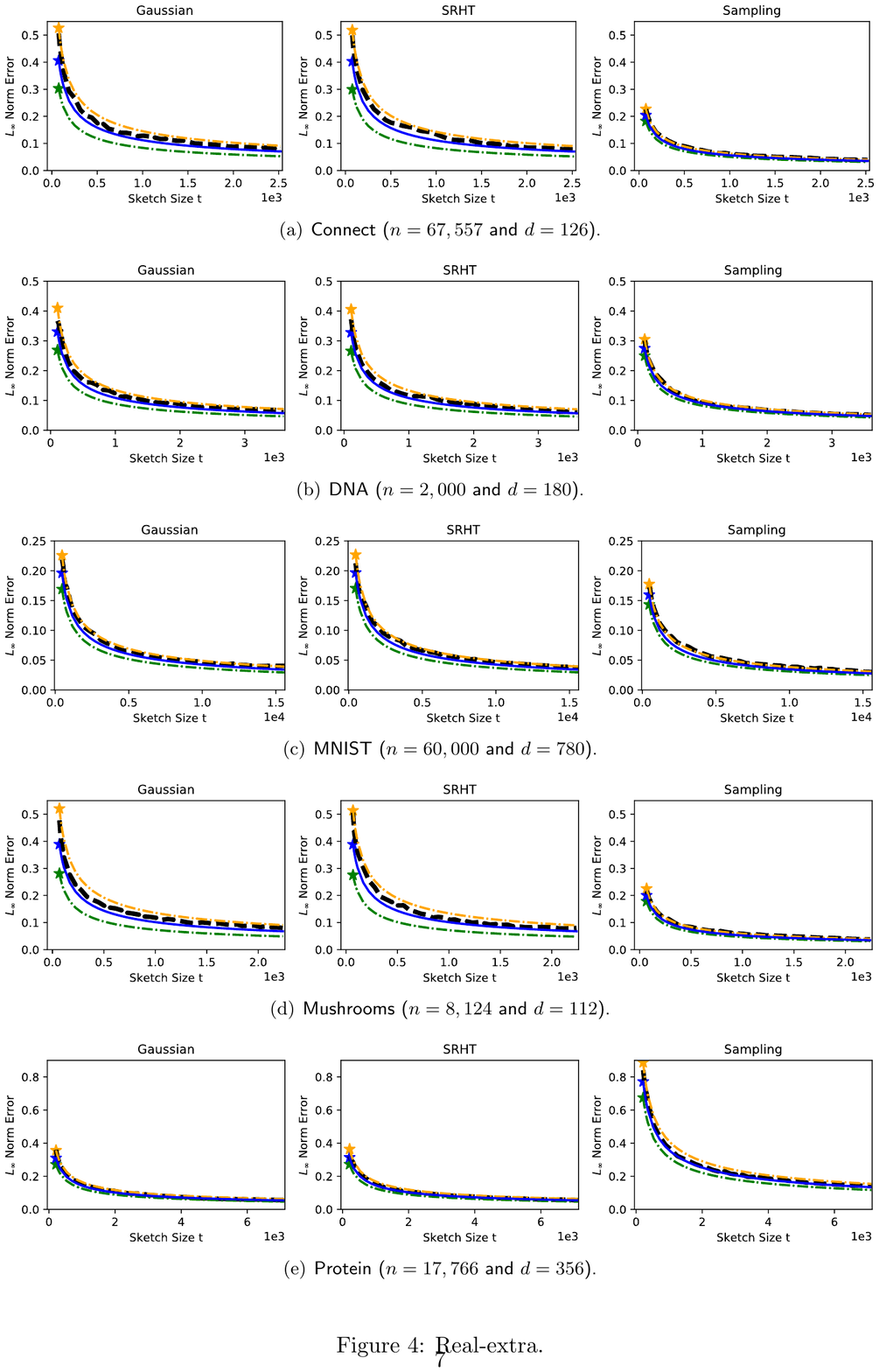}
	\end{center}
	\vspace{-8mm}
	\caption{{\textbf{Results for natural matrices.}}  The results for the natural matrices are plotted in the same way as described in the caption for the results on the synthetic matrices.
	}
	\label{fig:extra_real}
\end{figure}

\section{Conclusions and extensions}\label{sec:extend}
In this paper, we have focused on estimating the quantile $q_{1-\alpha}(t)$ as a way of addressing two fundamental  issues in randomized matrix multiplication: (1) knowing how accurate a given sketched product is, and (2) knowing how much computation is  needed to achieve a specified degree of accuracy. With regard to methodology, our approach is relatively novel in that it uses the statistical technique of bootstrapping to serve a computational purpose --- by quantifying the error of a randomized sketching algorithm. A second important component of our method is the extrapolation technique, which ensures that the cost of estimating $q_{1-\alpha}(t)$ does not substantially increase the overall cost of standard sketching methods. Furthermore, our numerical results show that the extrapolated estimate is quite accurate in a variety of different situations, suggesting that our method may offer a general way to enhance sketching algorithms in practice.

\paragraph{Extensions.} More generally, the problems we have addressed for randomized matrix multiplication arise for many other large-scale matrix computations. Hence, it is natural to consider extensions of our approach to more complex settings, and in the remainder of this section, we briefly mention a few possibilities for future study.

At a high level, each of the applications below deals with an object, say $\Theta$, that is difficult to compute, as well as a randomized approximation, say $\widetilde{\Theta}$, that is built from a sketching matrix $\S$ with $t$ rows. Next, if we consider the random error variable
$$\ve_t=\|\widetilde{\Theta}-\Theta\|,$$
for an unspecified norm $\|\cdot\|$, then the problem of estimating the relationship between accuracy and computation can again be viewed as the problem of estimating the quantile function
$q_{1-\alpha}(t)$ associated with $\ve_t$.  In turn, this leads to the question of how to develop a new bootstrap procedure that can generate approximate samples of $\ve_t$, yielding an estimate $\hat{q}_{1-\alpha}(t)$. However, instead of starting from the multiplier bootstrap (Algorithm~\ref{alg:bootstrap1}) as before, it may be conceptually easier to extend the non-parametric bootstrap (Algorithm~\ref{alg:bootstrap2}) --- because the latter bootstrap can viewed as a ``plug-in'' procedure that replaces $\A^T\B$ with $\tilde{\A}^T\tilde{\B}$, and replaces $\tilde{\A}^T\tilde{\B}$ with $(\tilde{\A}^*)^T(\tilde{\B}^*)$.

\begin{itemize}

\item \emph{Linear regression.}
Consider a multi-response linear regression problem, where the rows of $\B\in\mathbb{R}^{n\times d'}$
 are response vectors, and the rows of $\A\in\mathbb{R}^{n\times d}$ are input observations. The optimal solution to $\ell_2$-regression is given by
\begin{eqnarray*}
	\W_{\text{opt}} 
	\; = \; \argmin_{\W\in\mathbb{R}^{d\times d'}} \, \big\| \A \W - \B \big\|_F^2
	\; = \; (\A^T \A )^\dag \A^T \B ,
\end{eqnarray*}
which has $\OM (n d^2 + n d d')$ cost.
In the case where $\max\{d,d'\} \ll n$,  the matrix multiplications are a computational bottleneck, and an approximate solution can be obtained via 
$$\widetilde\W_{\text{opt}}=(\tilde{\A}^T \tilde{\A} )^\dag (\tilde{\A}^T \tilde{\B} ),$$ which has a cost $\OM (t d^2 + t d d') + C_{\textrm{sketch}}$,
where $C_{\textrm{sketch}}$ is cost of matrix sketching~\citep{drineas2006sampling,drineas2011faster,DMMW12_JMLR,clarkson2013low}. 
 In order to estimate the quantile function associated with the error variable
$\ve_t=\|\widetilde{\W}_{\text{opt}}-\W_{\text{opt}}\|,$
we could consider generating bootstrap samples of the form
$\ve_t^*=\|\widetilde{\W}_{\text{opt}}^*-\widetilde{\W}_{\text{opt}}\|,$
where 
$\widetilde{\W}_{\text{opt}}^*=( (\tilde{\A}^*)^T (\tilde{\A}^* ))^\dag (\tilde{\A}^*)^T (\tilde{\B}^*)$. For recent results in the case where $\W$ is a vector, we refer to the paper~\citep{LopesICML}.

\item \emph{Functions of covariance matrices.} If the rows of the matrix $\A$ are viewed as a sample of observations, then inferences on the population covariance structure are often based on functions of the form $\psi(\A^T\A)$. For instance, the function $\psi(\A^T\A)$ could be the top eigenvector, a set of eigenvalues, the condition number, or a test statistic. 
In any of these cases, if $\psi(\tilde{\A}^T\tilde{\A})$ is used as a fast approximation~\citep{covariancesketching}, then the sketching error $\ve_t=\|\psi(\tilde{\A}^T\tilde{\A})-\psi(\A^T\A)\|$ might be bootstrapped using $\ve_t^*=\|\psi((\tilde{\A}^*)^T(\tilde{\A}^*))-\psi(\tilde{\A}^T\tilde{\A})\|$.

\item \emph{Approximate Newton methods.}
In large-scale applications, Newton's method is often impractical, since it involves the costly processing of a Hessian matrix. 
As an example, consider an optimization problem arising in binary classification, where the rows of $\X\in\mathbb{R}^{n\times d}$ are observations $\x_1 , \dots , \x_n \in \RB^{d}$,
and $y_1, \dots , y_n \in \{0, 1\}$ are labels.
If an $\ell_2$-regularized logistic classifier is used,  this leads to minimizing the objective function $f(\w)=\sum_{i=1}^n \log \big( 1 + e^{- y_i \w^T \x_i} \big) + \frac{\gamma}{2} \|\w\|_2^2$ over coefficient vectors $\w$ in $\mathbb{R}^d$. The associated Newton step, with step size $\kappa$, is
$$ \w\leftarrow \w- \kappa\, \H^{-1}\nabla f,$$
involving the Hessian
\begin{eqnarray*}
	\H  = \A^T\A + \gamma \I_d ,
	\quad \textrm{ where \ \ \ } 
	\A = \diag \big( 1+e^{y_1 \w^T \x_1} , \dots ,1+e^{y_n \w^T \x_n} \big)^{-1}\X .
\end{eqnarray*}
If $d \ll n$, the cost of Newton's method is dominated by the formation of
$\H $ at each iteration, and the Hessian matrix can be approximated by the sketched version
	\smash{$\tilde{\H }= \tilde{\A}^T\tilde{\A} + \gamma \I_d$,}
which reduces the per-iteration cost from $\OM (n d^2)$
to \smash{$\OM (t d^2 + nd) + C_{\textrm{sketch}}$}
\citep{pilanci2015newton,roosta2016sub,xu2016sub}.  
 In this context, the quality of the approximate Newton step could be assessed in terms of the error
$$\ve_t=\|\tilde{\H}^{-1}\nabla f \ -\ \H^{-1}\nabla f\|,$$
and in turn, this might be bootstrapped using
$\ve_t^*=\|(\tilde{\H}^*)^{-1}\nabla f \ -\ \tilde{\H}^{-1}\nabla f\|,$
where $\tilde{\H}^*=(\tilde{\A}^*)^T(\tilde{\A}^*)+\gamma \I_d$.
\end{itemize}

\acks{We thank the anonymous reviewers for their helpful suggestions.
	MEL thanks the National Science Foundation for partial support under grant DMS-1613218. 
	MWM  would like to thank the National Science Foundation,  the Army Research Office, and the Defense Advanced Research Projects Agency for providing partial support of this work.
}

\section*{Appendices}

\noindent \textbf{Outline of appendices.}
Appendix~\ref{app:approx} explains the main conceptual ideas underlying the proofs of Theorems~\ref{thm:main} and~\ref{thm:main2}. In particular, the proofs of these theorems will be decomposed into two main results: Propositions~\ref{prop:gaussian} and~\ref{prop:bootstrap}, which are given in Appendix~\ref{app:approx}.

Appendix~\ref{app:prop:gaussian} will prove the sub-Gaussian case of Proposition~\ref{prop:gaussian}, and Appendix~\ref{app:prop:bootstrap} will prove the sub-Gaussian case of Proposition~\ref{prop:bootstrap}.  Later on, Appendices~\ref{app:case:sampling} and~\ref{app:case:srht}, will explain how the arguments can be changed to handle the length-sampling and SRHT cases.\\

\noindent \textbf{Conventions used in proofs.} If either of the matrices $\A$ or $\B$ are $\0$, then $\ve_t$ has a trivial point-mass distribution at 0. In this degenerate case, it is simple to check that the bootstrap produces an exact approximation. So, without loss of generality, all proofs are written under the assumption that $\A$ and $\B$ are non-zero. Next, since Assumption~\ref{assumptions} is formulated using the $\lesssim$ notation, there is no loss of generality in carrying out calculations under the assumption that all the numbers $t,n,d,d'$ are at least 8, which will ensure that quantities such as $\log(d)$ are greater than 2. Lastly, if a numbered lemma is invoked in the middle of a proof, the lemma may be found in Appendix~\ref{app:technical}.


\appendix


\section{Gaussian and bootstrap approximations} \label{app:approx}

Section~\ref{sec:proof:process} introduces some notation that helps us to analyze the rescaled sketching error $Z_t$ from the viewpoint of empirical processes. Next, in Section~\ref{sec:proof:lemmas}, Theorem~\ref{thm:main} will be decomposed into two propositions that compare $Z_t$ and $Z_t^{\star}$ with the maximum of a suitable Gaussian process. The proofs of these propositions may be found in Appendices~\ref{app:prop:gaussian} and~\ref{app:prop:bootstrap}.


\subsection{Making a link between empirical processes and sketching error} \label{sec:proof:process}

The main idea of our analysis is to view $Z_t$ as the maximum of an empirical process, 
which we now define. Recall the notation 
$$\D_i=\A^T\s_i\s_i^T \B \text{ \ \ and \ \  } \M_i \; = \; \A^T (\s_i \s_i^T - \I_n) \B.$$ 
Let $\GB_t(\cdot)$ be the empirical process that acts on  \emph{linear} functions $f: \RB^{d\times d'}\to \RB$, according to
\begin{equation*}
\GB_t(f) 
\; := \; \frac{1}{\sqrt{t}} \sum_{i=1}^t \Big(f(\D_i) - f(\A^T \B)\Big)
\; = \; \frac{1}{\sqrt{t}} \sum_{i=1}^t f(\M_i).
\end{equation*}
For future reference, we also define the corresponding bootstrap process
\begin{equation*}
\GB_t^{\star}(f)
\; := \; \frac{1}{\sqrt{t}} \sum_{i=1}^t \xi_i \cdot \Big( f(\D_i)-f \big(\A^T \S^T \S \B \big)\Big),
\end{equation*}
where $\xi_1,\dots,\xi_t$ are i.i.d.\ $\NM(0,1)$ and independent of $\S$.

Next, we define a certain collection $\FM$ of linear functions from $\RB^{d\times d'}$ to $\RB$.
Let $j_1 \in [d]$, $j_2\in[d']$, $s \in \{-1, 1\}$, and $\j :=(j_1,j_2,s)$. Then, for any matrix $\W\in\RB^{d\times d'}$, we put
\begin{equation*}
f_{\j}(\W ) 
\; := \; s \cdot \tr \big(\C_{\j}^T\W \big) ,
\end{equation*}
where $\C_{\j} :=  s \, \e_{j_1} \e_{j_2}^T \in \RB^{d\times d'}$
and $\e_{j_1} \in \RB^d$, \ $\e_{j_2}\in\RB^{d'}$ are standard basis vectors. In words, the function $f_{\j}$ merely picks out the $(j_1,j_2)$ entry of $\W$, and multiplies by a sign $s$.
Likewise, let $\JM$ be the collection of all the triples $\j$, and define the class of linear functions
$$\FM \; := \big\{f_{\j}\ | \ \j \in \JM \big\}.$$
Clearly, $\card (\FM ) = 2 dd'$. 
Under this definition, it is simple to check that $Z_t$ and $Z_t^{\star}$, defined in equations \eqref{ztdef} and \eqref{ztstardef}, can be expressed as
\begin{eqnarray*}
	Z_t \; = \;  \max_{f_{\j}\in\FM}\, \GB_t(f_{\j}), 
	\qquad \textrm{and} \qquad
	Z_t^{\star} \; = \; \max_{f_{\j}\in\FM} \, \GB_t^{\star}(f_{\j}).
\end{eqnarray*}


\subsection{Statements of the approximation results} \label{sec:proof:lemmas}

Theorems~\ref{thm:main} and~\ref{thm:main2} are obtained by combining the following two results (Propositions~\ref{prop:gaussian} and~\ref{prop:bootstrap}) via the triangle inequality. In essence, these results are based on a comparison with the maximum of a certain Gaussian process. 
More specifically, let $\GB:\FM \to \RB$ be a zero-mean Gaussian process 
whose covariance structure is defined according to
\begin{align}\label{eq:proof:gaussiancovariance}
 \EB \big [\GB(f_{\j}) \, \GB(f_{\k}) \big] 
& = \; \cov \big( f_{\j}(\D_1), \, f_{\k}(\D_1) \big) \nonumber \\
& = \; \EB \Big[f_{\j}(\D_1) \, f_{\k}(\D_1)\Big]- f_{\j}(\A^T \B) \, f_{\k}(\A^T \B),
\end{align}
for all $\j,\k\in\JM$.
In turn, define the following random variable as the the maximum of this Gaussian process,
\begin{equation*}
Z \; := \; \max_{f_{\j} \in \FM} \, \GB(f_{\j}).
\end{equation*}

In order to handle the case of SRHT matrices,  define another zero-mean Gaussian process $\tilde{\GB}:\FM\to\RB$ (conditionally on a fixed realization of $\D_n^{\circ})$ to have its covariance structure given by
\begin{align}\label{eq:proof:gaussianconditionalcovariance}
 \EB \big [\tilde{\GB}(f_{\j}) \, \tilde{\GB}(f_{\k}) \big| \D_n^{\circ} \big] 
& = \; \cov \big( f_{\j}(\D_1), \, f_{\k}(\D_1) \big| \D_n^{\circ}\big) \nonumber \\
& = \; \EB \Big[f_{\j}(\D_1) \, f_{\k}(\D_1)\Big| \D_n^{\circ}\Big]- f_{\j}(\A^T \B) \, f_{\k}(\A^T \B),
\end{align}
and let $\tilde{Z}$ denote the maximum of the process $\tilde{\GB}$,
\begin{equation*}
\tilde{Z} \; := \; \max_{f_{\j} \in \FM} \, \tilde{\GB}(f_{\j}).
\end{equation*}

\noindent We are now in position to state the approximation results.

\begin{proposition}[Gaussian approximation]\label{prop:gaussian}
	Under Assumption~\ref{assumptions} (a), the following bound holds,
	\begin{equation*}
		d_{\text{\emph{LP}}}\big(\mathcal{L}(Z_t)\, , \,\mathcal{L}(Z)\big) 
	\; \leq \; \frac{ c \cdot \nu(\A,\B)^{3/4}\cdot \sqrt{\log(d)} }{ t^{1/8} }.
	\end{equation*}
	Under Assumption~\ref{assumptions} (b), the following bound holds,
	\begin{equation*}
		d_{\text{\emph{LP}}}\big(\mathcal{L}(Z_t)\, , \,\mathcal{L}(Z)\big) 
	\; \leq \; \frac{ c \cdot (\|\A\|_F\|\B\|_F)^{3/4}\cdot \sqrt{\log(d)} }{ t^{1/8} }.
	\end{equation*}	
		Under Assumption~\ref{assumptions} (c), the following bound holds with probability at least $1-c/n$
	\begin{equation*}
		d_{\text{\emph{LP}}}\big(\mathcal{L}(Z_t)\, , \,\mathcal{L}(\tilde{Z}|\D_n^{\circ})\big) 
	\; \leq \; \frac{ c \cdot \nu(\A,\B)^{3/4}\cdot (\log(n))^{3/4}\cdot \sqrt{\log(d)} }{ t^{1/8} }.
	\end{equation*}
\end{proposition}

\begin{proposition}[Bootstrap approximation]\label{prop:bootstrap}
	If Assumption~\ref{assumptions} (a) holds, then the following bound holds with probability at least $1-\ts\frac{1}{t}-\ts\frac{1}{dd'}$,
	\begin{equation*}
		d_{\text{\emph{LP}}}\big(\mathcal{L}(Z)\, , \, \mathcal{L}(Z_t^{\star}|\S)\big) 
	\; \leq \;  \frac{ c\cdot\nu(\A,\B)^{1/2}\sqrt{\log(d)} }{ t^{1/8} }.
	\end{equation*}
	If Assumption~\ref{assumptions} (b) holds, then the following bound holds with probability at least $1-\ts\frac{1}{t}-\ts\frac{1}{dd'}$,
	\begin{equation*}
		d_{\text{\emph{LP}}}\big(\mathcal{L}(Z)\, , \, \mathcal{L}(Z_t^{\star}|\S)\big) 
	\; \leq \;  \frac{ c\cdot(\|\A\|_F\|\B\|_F)^{1/2}\sqrt{\log(d)} }{ t^{1/8} }.
	\end{equation*}	If Assumption~\ref{assumptions} (c) holds, then the following bound holds with  probability at least \smash{$1-\ts\frac{1}{t}-\ts\frac{1}{dd'}-\ts\frac{c}{n}$,}
	\begin{equation*}
		d_{\text{\emph{LP}}}\big(\mathcal{L}(\tilde{Z}|\D_n^{\circ})\, , \, \mathcal{L}(Z_t^{\star}|\S)\big) 
	\; \leq \;  \frac{ c\cdot\nu(\A,\B)^{1/2}\cdot \log(n)^{1/2}\cdot\sqrt{\log(d)} }{ t^{1/8} }.
	\end{equation*}
\end{proposition}


\section{Proof of Proposition~\ref{prop:gaussian}, part (a)} \label{app:prop:gaussian}

Let $\AM \subset\RB$ be a Borel set. Due to Theorem~3.1 from the paper \citet{cckIncreasing},  we have for any $\delta>0$,
\begin{equation}\label{discretecomparison}
\PB(Z_t \in \AM) \leq \PB(Z\in \AM^{ c\delta})+  \ts\frac{c \log^2(d)}{\delta^3 \sqrt{t}}\Big(L_t+K_t(\delta)+J_t(\delta))\Big),
\end{equation}
where we define the following non-random quantities 
\begin{eqnarray}
L_t  
& := & \max_{f_{\j}\in \FM} \frac{1}{t} \sum_{i=1}^t \EB \Big[ \big|f_{\j}(\M_i) \big|^3 \Big]\label{ltdef} , \\
K_t(\delta)
&:= & \EB\Bigg[ \max_{f_{\j}\in\FM} |f_{\j}(\M_1)|^3  
\cdot 1\Big\{\max_{f_{\j}\in\FM}|f_{\j}(\M_1)|>\ts\frac{\delta \sqrt{t}}{\log(\text{card}(\FM))}\Big\}\Bigg]\label{ktdef} ,\\
J_t(\delta )
&:=& \EB\Bigg[ \max_{f_{\j}\in\FM} |\GB(f_{\j})|^3  
\cdot 1\Big\{\max_{f_{\j}\in\FM}|\GB(f_{\j})|>\ts\frac{\delta \sqrt{t}}{\log(\text{card}(\FM))}\Big\}\Bigg]\label{jtdef}.
\end{eqnarray}
 The remainder of the proof consists in bounding each of these quantities, and we will establish the following two bounds for all $\delta>0$,
\begin{eqnarray}
L_t 
& \leq & c\,\nu(\A,\B)^3,
\label{lbound} \\
K_t(\delta) +J_t(\delta)
& \leq &  c \Big(\ts\frac{\delta \sqrt{t}}{\log(d)}
+\log(d) \,\nu(\A,\B)\Big)^3 \cdot
\exp \Big(-\ts\frac{\delta \sqrt{t}}{c \ \nu(\A,\B) \log^2(d)} \Big).
\label{jkbound}
\end{eqnarray}
Recall also that $\text{card}(\FM)=2dd'$, and  $d\asymp d'$ under Assumption~\ref{assumptions}.

For the moment, we set aside the task of proving these bounds, and consider the choice of $\delta$. 
There are two constraints that we would like $\delta$ to satisfy. 
First, we would like to choose $\delta$ so that the bounds on $L_t$
 and $(K_t(\delta)+J_t(\delta))$ are of the same order. In particular, we desire
\begin{equation}\label{firstconstraint}
K_t(\delta)+J_t(\delta)
\; \leq \; c\, \nu(\A,\B)^3.
\end{equation}
Second, with regard to line~\eqref{discretecomparison} we would like $\delta$ to solve the equation
\begin{equation}\label{solvedelta}
\delta
\; = \; \ts\frac{1}{\delta^3}\ts\frac{ \log^2(d)\,\nu(\A,\B)^3}{ \sqrt{t}},
\end{equation}
so that the second term in line~\eqref{discretecomparison} is of order $\delta$.
The idea is that if $\delta$ satisfies both of the conditions~\eqref{firstconstraint} and~\eqref{solvedelta}, then the definition of the $d_{\text{LP}}$ metric and line~\eqref{discretecomparison} imply
\begin{equation*}
d_{\text{LP}}(\mathcal{L}(Z),\mathcal{L}(Z_t)) 
\; \leq \; c\, \delta. 
\end{equation*}
To proceed, consider the choice 
\begin{equation*}
\delta_0 \; := \; \ts\frac{\log^{1/2}(d) \,\nu(\A,\B)^{3/4}}{ t^{1/8}},
\end{equation*}
which clearly satisfies line~\eqref{solvedelta}. 
Futhermore, it can be checked that $\delta_0$ also satisfies the constraint~\eqref{firstconstraint} 
under Assumption~\ref{assumptions} (a). 
(The details of verifying this are somewhat tedious and are given in Lemma~\ref{lem:verify} in Appendix~\ref{app:technical}.) 

To finish the proof, it remains to establish the bounds~\eqref{lbound} and~\eqref{jkbound}. To handle $L_t$, note that\footnote{In this step, we use the assumption that $\|\sqrt{t}S_{i,j}\|_{\psi_2}\leq c$ for all $i$ and $j$.}
\begin{eqnarray}
\EB[|f_{\j}( \M_i)|^3] 
& = & \|f_{\j}(\M_1)\|_3^3 
\nonumber  \\[0.2cm]
&\leq & c \|f_{\j}(\M_1)\|_{\psi_1}^3, \ \ \  \ \ \ \  \ \ \ \  \ \ \ \  \text{ (Lemma~\ref{lem:orlicz})}
\label{fjorlicz} \\[0.2cm]
& \leq & c\Big(\ts\frac{\|\B \C_{\j}^T \A^T \|_{F}^2}{\|\B \C_{\j}^T \A^T\|_2}\Big)^{3}, 
\ \ \ \ \ \ \ \ \ \ \ \  \text{ (Lemma~\ref{lem:quadnorm})} 
\nonumber \\[0.2cm]
&=&  c\Big(\ts\|\B \C_{\j}^T \A^T \|_F\Big)^{3} \ \ \  \ \ \ \ \ \ \ \  \text{ (since $\|\H\|_2=\|\H\|_F$ when $\H$ is rank-1)}\nonumber\\[0.2cm]
&=& c\Big(\tr\big(\B \C_{\j}^T \A^T \A \C_{\j} \B^T\big)\Big)^{3/2}\nonumber\\[0.2cm]
& = & c\Big(\e_{j_1}^T\A^T\A \e_{j_1}\cdot\e_{j_2}^T \B^T\B \e_{j_2} \Big)^{3/2}\nonumber\\[0.2cm]
&\leq &c\ \nu(\A,\B)^3,
\label{fjorliczend} 
\end{eqnarray}
which proves the claimed bound in line~\eqref{lbound}.\\

Next, regarding $K_t(\delta)$, let us consider the random variable 
$$\eta:=\max_{f_{\j}\in\FM} |f_{\j}(\M_1)|.$$
It follows from Lemma~\ref{lem:orlicz} (part \ref{orlicztail}) and Lemma~\ref{lem:jkbounds} in Appendix~\ref{app:technical} that $K_t(\delta)$ can be bounded in terms of the Orlicz norm $\|\eta\|_{\psi_1}$,
\begin{equation*}
\begin{split}
K_t(\delta) & \leq c\Big(\ts\frac{\delta \sqrt{t}}{\log(\card(\FM))}+\|\eta\|_{\psi_1}\Big)^3\cdot\exp(-\ts\frac{\delta \sqrt{t}}{\|\eta\|_{\psi_1}\, \log(\card(\FM)}).
\end{split}
\end{equation*}
To handle $\|\eta\|_{\psi_1}$, it follows from Lemma~\ref{lem:orlicz} (part~\ref{partmax}), that
\begin{equation}\label{psibound}
\| \eta \|_{\psi_1} \leq c \log(\card(\FM)) \cdot \max_{f_{\j}\in\FM}\|f_{\j}(\M_1)\|_{\psi_1}.
\end{equation}
Furthermore, due to the earlier calculation starting at line~\eqref{fjorlicz} above, 
\begin{equation}\label{psinormquadform}
\begin{split}
\|f_{\j}(\M_1)\|_{\psi_1} & \leq c\ \nu(\A,\B).
\end{split}
\end{equation}
 Combining the last few steps, we conclude that
\begin{equation}\label{kbound}
\begin{split}
K_t(\delta) 
& \leq c\Big(\ts\frac{\delta \sqrt{t}}{\log(\text{card}(\FM))}+\log(\text{card}(\FM))\nu(\A,\B)\Big)^3\cdot
\exp \Big(-\ts\frac{\delta \sqrt{t}}{c \nu(\A,\B) \log^2(\text{card}(\FM))} \Big).
\end{split}
\end{equation}

Lastly, we turn to bounding $J_t(\delta)$. Fortunately, much of the argument for bounding $K_t(\delta)$ can be carried over. Specifically, consider the random variable  
$$\zeta:=\max_{f_{\j}\in\FM} |\GB (f_{\j})|.$$
Lemma~\ref{lem:jkbounds} in Appendix~\ref{app:technical} shows that $J_t(\delta)$ can be bounded in terms of $\|\zeta\|_{\psi_1}$,
\begin{equation*}
J_t(\delta) \leq  c\Big(\ts\frac{\delta \sqrt{t}}{\log(\text{card}(\FM))}+\|\zeta\|_{\psi_1}\Big)^3
\cdot \exp \Big(-\ts\frac{\delta \sqrt{t}}{\|\zeta\|_{\psi_1}\, \log(\text{card}(\FM))} \Big).
\end{equation*}
Proceeding in a way
 that is similar to the bound for $K_t(\delta)$, it follows from part~(\ref{partmax}) of Lemma~\ref{lem:orlicz} that
\begin{equation*}
\begin{split}
\|\zeta\|_{\psi_1}&\leq c\log(\card(\FM)) \cdot \max_{f_{\j}\in\FM} \|\GB (f_{\j})\|_{\psi_1}. \\
\end{split}
\end{equation*}
Furthermore, for every $f_{\j}\in\FM$, the facts in Lemma~\ref{lem:orlicz}  imply
\begin{eqnarray}
\|\GB (f_{\j})\|_{\psi_1} &\leq & c\, \|\GB (f_{\j})\|_{\psi_2}\nonumber\\[0.2cm]
&\leq & c \sqrt{\var(\GB (f_{\j}))}\nonumber\\[0.2cm]
& = & c\sqrt{\var(f_{\j}(\D_1))} \ \ \text{ (by definition of $\mathbb{G}$)}\nonumber\\[0.2cm]
&\leq & c\|f_{\j}(\D_1)\|_2\nonumber \\[0.2cm]
&\leq & c\|f_{\j}(\D_1)\|_{\psi_1} \ \ \label{fjd1}\\[0.2cm]
&\leq & c\, \|f_{\j}(\D_1)-\EB[f_{\j}(\D_1)\|_{\psi_1} +c\,|\EB[f_{\j}(\D_1)]|\nonumber\\[0.2cm]
&= & c\|f_{\j}(\M_1)\|_{\psi_1}+c\,|\tr(\B\C_{\j}^T\A)|\nonumber\\[0.2cm]
&\leq  & c \; \nu(\A,\B), \label{endfjd1}
\end{eqnarray}
where the last step follows from the bounds~\eqref{fjorlicz} through~\eqref{fjorliczend}, and the fact that $|\tr(\B\C_{\j}^T\A)|\leq \nu(\A,\B)$.
Consequently, up to a constant factor, $J_t(\delta)$ satisfies the same bound as $K_t(\delta)$ given in line~\eqref{kbound}, and this proves the claim in line~\eqref{jkbound}.

\hfill \BlackBox


\section{Proof of Proposition~\ref{prop:bootstrap}, part (a)} \label{app:prop:bootstrap}

We will show there is a set of ``good'' sketching matrices 
$\SM\subset \RB^{t \times n}$ with the following two properties. 
First, a randomly drawn sketching matrix $\S$ is likely to fall in $\SM$. Namely,
\begin{equation}\label{firstprop}
\PB(\S\in\SM)\geq 1-\ts\frac{1}{t}.
\end{equation}
Second, whenever the event $\{\S\in \SM\}$ occurs, we have the following bound for any $\delta>0$ 
and any Borel set $\AM \subset \RB$,
\begin{equation}\label{secondprop}
\PB
\Big(\max_{f_{\j}\in \FM} \GB_t^{\star}(f_{\j})\in \AM \Big| \S\Big) 
\; \leq \; \PB\Big(\max_{f_{\j}\in\FM}\GB(f_\j)\in \AM^{\delta}\Big)
+ \ts\frac{c\,\nu(\A,\B)\cdot\, \log(\card (\FM))}{\delta \, t^{1/4}}.\\[0.2cm]
\end{equation}
If we set $\delta$ to the particular choice
$\delta_0:=t^{-1/8}\sqrt{\nu(\A,\B)\cdot\,\log(\card(\FM))}$, 
then $\delta_0$ solves the equation
$$\delta_0
\; = \; \ts\frac{\nu(\A,\B)\cdot \,\log(\card(\FM))}{\delta_0 \, t^{1/4}}.$$
Consequently, by the definition of the $d_{\text{LP}}$ metric, 
this implies that whenever the event $\{\S\in\SM\}$ occurs, we have
\begin{equation}
d_{\text{LP}}\big(\mathcal{L}(Z_t^{\star}|\S)\, , \, \mathcal{L}(Z)\big) 
\; \leq \;  c\, t^{-1/8}\,\sqrt{\nu(\A,\B)\cdot\log(\card(\FM))},
\end{equation}
and this implies the statement of Proposition~\ref{prop:bootstrap}.

To proceed with the main argument of constructing $\SM$ and demonstrating the two properties~\eqref{firstprop} and~\eqref{secondprop}, it is helpful to think of $\GB_t^{\star}$ (conditionally on $\S$) 
and $\GB$ as Gaussian vectors of dimension $\card(\FM)=2dd'$. 
From this point of view, we can compare the maxima of these vectors 
using a result due to~\citet[Theorem 3.2]{cckIncreasing}.
Under our assumptions, this result implies that for any realization of $\S$, any number $\delta>0$, 
and any Borel set $\AM\subset \RB$, we have
\begin{eqnarray*}
\PB\Big(\max_{f_{\j}\in \FM} 
\GB_t^{\star}(f_{\j})\in \AM \Big| \S\Big)  
& \leq & \PB\Big(\max_{f_{\j}\in\FM}\GB(f_\j)\in \AM^{\delta}\Big)+\ts\frac{c\sqrt{ \Delta_t (\S) \log(\card(\FM))}}{\delta},
\end{eqnarray*}
where we define the following function of $\S$,
\begin{eqnarray}\label{deltadef}
\Delta_t(\S)
& := &
\max_{(f_{\j},f_{\k})\in\FM\times \FM} \bigg|
\EB \big[\GB_t^{\star}(f_{\j})\GB_t^{\star}(f_{\k})\big| \S\big] 
-\EB \big[\GB(f_{\j})\GB(f_{\k})\big] \bigg|.
\end{eqnarray}
When referencing Theorem 3.2 from the paper~\citet{cckIncreasing}, 
note that $\EB [\GB(f_{\j})]=0$ and $\EB [\GB_t^{\star}(f_{\j})| \S]=0$ 
for all $f_{\j}\in\FM$. 
To interpret $\Delta_t(\S)$, 
it may be viewed as the $\ell_{\infty}$-distance between 
the covariance matrices associated with $\GB_t^{\star}$ (conditionally on $\S$) and $\GB$.

Using the above notation, we define the set of sketching matrices $\SM\subset\RB^{n \times t}$ according to
\begin{equation}\label{sdef}
\S\in\SM  \qquad
\text{ if and only if } \qquad 
\Delta_t (\S)\leq \ts\frac{c}{\sqrt{t}}\cdot \nu(\A,\B)^2\cdot \log(\card (\FM)).
\end{equation} 
Based on this definition, it is simple to check that the proof is reduced to showing that the event $\{\S\in\SM\}$ occurs with probability at least $1-\ts\frac{1}{t}-\ts\frac{1}{dd'}$. This is guaranteed by the lemma below.

\hfill \BlackBox

\begin{lemma}
	Suppose Assumption~\ref{assumptions} (a) holds. 
	Then, the event 
	\begin{equation*}
	\Delta_t(\S)
	\; \leq  \; \frac{c}{\sqrt{t}}\cdot \nu(\A,\B)^2\cdot \log \big( \card (\FM) \big)
	\end{equation*}
	occurs with probability at least $1-\ts\frac{1}{t}-\ts\frac{1}{dd'}$.
\end{lemma}

\begin{proof}
We begin by bounding $\Delta_t (\S)$ with two other quantities 
(to be denoted $\Delta_t'(\S)$, $\Delta_t''(\S)$) that are easier to bound. Using the fact that $\ts\frac{1}{t}\sum_{i=1}^t \D_i=\A^T\S^T\S \B$ it can be checked that
\begin{eqnarray*}
\EB\big[\GB_t^{\star}(f_{\j})\GB_t^{\star}(f_{\k})\big| \S\big] 
& = & \Big(\ts\frac{1}{t}\tsum_{i=1}^t f_{\j}(\D_i)f_{\k}(\D_i)\Big) 
-\Big( \ts\frac{1}{t}\tsum_{i=1}^t f_{\j}(\D_i)\Big) 
\cdot \Big( \ts\frac{1}{t}\tsum_{i=1}^t f_{\k}(\D_i)\Big).
\end{eqnarray*}
Similarly, recall from line~\eqref{eq:proof:gaussiancovariance} that
\begin{equation*}
\EB[\GB(f_{\j})\GB(f_{\k})] 
\; = \; \EB\Big[f_{\j}(\D_1) f_{\k}(\D_1)\Big]
- \EB[f_{\j}(\D_1)]\cdot \EB[f_{\k}(\D_1)].
\end{equation*}
From looking at the last two lines, it is natural to define the following \emph{zero-mean} random variables for any triple $i,\j,\k$,\footnote{Note that $Q_{i,\j,\k}$ is a multivariate polynomial 
	of degree-4 in the variables $S_{i,j}$, 
	and so techniques based on moment generating functions, like Chernoff bounds, 
	are not generally applicable to controlling $Q_{i,\j,\k}$. 
	For instance, if $X\sim \NM(0,1)$, then the variable $X^4$ does not have a moment generating function. Handling this obstacle is a notable aspect of our analysis.} 
\begin{equation*}
Q_{i,\j,\k}:=f_{\j}(\D_i)f_{\k}(\D_i) - \EB\Big[f_{\j}(\D_i)f_{\k}(\D_i)\Big],
\end{equation*}
and
\begin{equation*}
R_{t,\j}:=\ts\frac{1}{t}\tsum_{i=1}^t \big( f_{\j}(\D_i) -\EB[f_{\j}(\D_i)]\big).
\end{equation*}
Then, some algebra shows that
\small
\begin{align*}
 \EB\big[\GB_t^{\star}(f_{\j}) \, \GB_t^{\star}(f_{\k}) \, \big| \, \S\big] 
-  \EB \big[ \GB(f_{\j}) \, \GB(f_{\k}) \big] &=
  \;  \Big(\ts\frac{1}{t} \sum_{i=1}^t Q_{i,\j,\k}\Big)
- R_{t,\j}\cdot R_{t,\k}\\[0.2cm]
&\ \ \ \ \ \ \ \ -\EB \big[f_{\j}(\D_1) \big]\cdot R_{t,\k}- \EB \big[f_{\k}(\D_1) \big]\cdot R_{t,\j}.
\end{align*}
\normalsize
So, if we define the quantities
\begin{align*}
\Delta_t'(\S)
& \; := \; \max_{(\j,\k)\in \JM\times \JM}\bigg| \ts\frac{1}{t}\displaystyle\sum_{i=1}^t Q_{i,\j,\k}\bigg|,\\[0.4cm]
\Delta_t''(\S)
& \; := \; \max_{\j\in\JM}\big| R_{t,\j}\big|,
\end{align*}
then  
\begin{equation*}
\Delta_t(\S) 
\; \leq \;  \Delta_t'(\S)+\Delta_t''(\S)^2+2\nu(\A,\B) \cdot \Delta_t''(\S),
\end{equation*}
where we have made use of the simple bound $|\EB[f_{\j}(\D_1)]|\leq \|\A^T\B\|_{\infty}\leq \nu(\A,\B)$.
The following lemma establishes tail bounds for $\Delta_t'(\S)$ and $\Delta_t''(\S)$, 
which lead to the statement of Proposition~\ref{prop:bootstrap}.
\end{proof}

\begin{lemma}\label{lem:deltaprime}
	Suppose Assumption~\ref{assumptions} (a) holds. 
	Then, the event
	\begin{equation}\tag{i}\label{eqn:deltaprime}
	\Delta'_t(\S)
	\; \leq \; \frac{c}{\sqrt{t}}
	\cdot \nu(\A,\B)^2
	\cdot \log \big(\card(\FM) \big)
	\end{equation}
	occurs with probability at least $1- \frac{1}{t}$, and  the event
	\begin{equation}\tag{ii}\label{eqn:deltadoubleprime}
	\Delta_t''(\S) 
	\; \leq \; 
	\frac{c}{\sqrt{t}}
	\cdot \nu(\A,\B)
	\cdot \sqrt{\log \big( \card(\FM) \big)}
	\end{equation}
	occurs with probability at least $1- \frac{1}{dd'}$.
\end{lemma}

\paragraph{Proof of Lemma~\ref{lem:deltaprime} (i).}

Let $p>2$. Due to part~(\ref{partmax}) of Lemma~\ref{lem:orlicz} in Appendix~\ref{app:technical}, we have

\begin{equation}\label{kmaxbound}
\| \Delta_t'(\S)\|_p
\; \leq \; \big(\card(\FM )^2\big)^{1/p} 
\cdot \max_{(\j,\k)\in\JM\times \JM} \bigg\|\ts\frac{1}{t}\tsum_{i=1}^t Q_{i,\j,\k}\bigg\|_p.
\end{equation}
Note that each variable $Q_{i,\j,\k}$ has moments of all orders, and when $\j$ and $\k$ are held fixed, the sequence  $\{Q_{i,\j,\k}\}_{1\leq i\leq t}$ is i.i.d. For this reason, it is natural to use Rosenthal's inequality to bound the $L_p$ norm of the right side of the previous line. Specifically, the version of Rosenthal's inequality\footnote{Here we are using the version of Rosenthal's inequality with the optimal dependence on $p$. It is a notable aspect of our argument that it makes essential use of this scaling in $p$.} stated in Lemma~\ref{lem:rosenthal} in Appendix~\ref{app:technical} leads to
\begin{equation}\label{rosenthalbound}
\Big\|\ts\frac{1}{t}\tsum_{i=1}^t Q_{i,\j,\k}\Big\|_p \leq c\cdot \ts\frac{p/\log(p)}{t} \cdot\max\bigg\{ \Big\|\tsum_{i=1}^t Q_{i,\j,\k}\Big\|_2 \, , \, \big(\tsum_{i=1}^t \|Q_{i,\j,\k}\|_p^p\big)^{1/p}\bigg\}.
\end{equation}
The $L_2$ norm on the right side of Rosenthal's inequality~\eqref{rosenthalbound} satisfies the bound
\begin{eqnarray*}
\Big\|\tsum_{i=1}^t Q_{i,\j,\k}\Big\|_2  
& = & \sqrt{\var(\tsum_{i=1}^t Q_{i,\j,\k})}\\[0.2cm]
& = & \sqrt{t} \sqrt{\var(Q_{1,\j,\k})}\\[0.2cm]
& = & \sqrt{t}\sqrt{\var\big(f_{\j}(\D_1) \, f_{\k}(\D_1)\big)} \\[0.2cm]
& \leq & \sqrt{t}\Big\|f_{\j}(\D_1) \, f_{\k}(\D_1)\Big\|_2 \\[0.2cm]
& \leq & \sqrt{t} \big\|f_{\j}(\D_1)\big\|_4
	\cdot \big\|f_{\k}(\D_1)\big\|_4 \qquad \text{(Cauchy-Schwarz)} \\[0.2cm]
& \leq & c \sqrt{t} \big\| f_{\j}(\D_1)\big\|_{\psi_1}
	\cdot \big\|f_{\k}(\D_1)\big\|_{\psi_1} \qquad \text{(Lemma~\ref{lem:orlicz})}\\[0.2cm]
&\leq & c \; \sqrt{t} \nu(\A,\B)^2,
\end{eqnarray*}
where the last step follows from the fact
\begin{equation}\label{eqn:fjpsi1_bound}
\|f_{\j}(\D_1)\|_{\psi_1} \leq c\, \nu(\A,\B),
\end{equation}
obtained in the bounds~\eqref{fjorlicz} through~\eqref{fjorliczend}.

Next, to handle the $L_p$ norms in the bound~\eqref{rosenthalbound}, observe that 
\begin{eqnarray*}
\|Q_{1,\j,\k}\|_p 
& \leq & \big\|f_{\j}(\D_1 ) f_{\k}(\D_1)\big\|_p +\big|\EB\big[f_{\j}(\D_1)f_{\k}(\D_1)\big]\big| \\[0.2cm]
&\leq & 2 \, \big\|f_{\j}(\D_1)\big\|_{2p}\cdot \big\|f_{\k}(\D_1)\big\|_{2p} \qquad \ \ \  \text{(Cauchy-Schwarz)}\\[0.2cm]
&\leq & c\, p^2 \, \big\|f_{\j}(\D_1)\big\|_{\psi_1}\cdot \big\|f_{\k}(\D_1)\big\|_{\psi_1} 
	\ \ \  \ \ \text{(Lemma~\ref{lem:orlicz} in Appendix~\ref{app:technical})}\\[0.2cm]
&\leq & c\,p^2\, \nu(\A,\B)^2 \qquad\qquad \text{ \ \    \ \ \ \ \  \ \ \ \ \ \  (inequality~\eqref{eqn:fjpsi1_bound})}.
\end{eqnarray*}
Hence, the second term in the Rosenthal bound~\eqref{rosenthalbound} satisfies
\begin{equation*}
\big(\tsum_{i=1}^t \|Q_{i,\j,\k}\|_p^p\big)^{1/p} 
\; \leq \; c\cdot p^2\cdot t^{1/p}\cdot \nu(\A,\B)^2,
\end{equation*}
and as long as the first term in the Rosenthal bound dominates\footnote{Under the choice of $p=\log(\card(\FM))=\log(2dd')$ that will be made at the end of this argument, it is straightforward to check  that the condition~\eqref{tkcondition} holds under Assumption~\ref{assumptions}.}, i.e. 
\begin{equation}\label{tkcondition}
p^2\ t^{1/p} \lesssim \, t^{1/2} 
\end{equation}
then we conclude that for any $\j$ and $\k$,
\begin{equation*}
\Big\| \ts\frac{1}{t}\tsum_{i=1}^t Q_{i,\j,\k} \Big\|_p 
\; \leq \; \ts\frac{c\cdot(p/\log(p))\cdot \nu(\A,\B)^2}{\sqrt{t}} .
\end{equation*}

Since the previous bound does not depend on $\j$ or $\k$, combining it with the first step in line~\eqref{kmaxbound} leads to
\begin{equation*}
\big\|\Delta_t'(\S)\big\|_p 
\; \leq \;  c\cdot (p/\log(p))\cdot \card(\FM )^{2/p}
\cdot \ts\frac{\nu(\A,\B)^2}{\sqrt{t}} .
\end{equation*}
Next, we convert this norm bound 
into a tail bound. Specifically, if we consider the value
\begin{eqnarray*}
 x_p
\; := \; c \cdot  (p/\log(p)) \cdot \card(\FM)^{2/p}\cdot \ts\frac{\nu(\A,\B)^2}{\sqrt{t}} \cdot t^{1/p}
\end{eqnarray*}
then Markov's inequality gives 
\begin{equation*}
\PB \big( \Delta'_t(\S) \, \geq \,  x_p \big) \; \leq \; \ts\frac{\|\Delta'_t(\S)\|_p^p}{x_p^p}\leq \ts\frac{1}{t}.
\end{equation*}
Considering the choice of $p$ given by
\begin{equation*}
p=\log(\card(\FM)),
\end{equation*}
and noting that $\card(\FM)^{1/p}= e$, it follows that under this choice of $p$,
\begin{eqnarray*}
x_p 
\; \leq \; \Big(\ts\frac{ c\cdot \nu(\A,\B)^2
\cdot \log(\card(\FM))}{\sqrt{t}}\Big)
\cdot \Big(\ts\frac{t^{1/p}}{\log(p)}\Big).
\end{eqnarray*}
Moreover, as long as $t\lesssim \card(\FM)^{\kappa}$ for some absolute constant $\kappa \geq1$ (which holds under Assumption~\ref{assumptions}), then the last factor on the right satisfies
$$\big(\ts\frac{t^{1/p}}{\log(p)}\big) \leq \ts\frac{(\card(\FM)^{1/p})^{\kappa}}{\log(p)} = \ts\frac{e^{\kappa}}{\log(\log(\card(\FM))}\lesssim 1.$$
So, combining the last few steps, there is an absolute constant $c$ such that
$$\PB \Big( \Delta'_t(\S) \, \geq \, \ts \frac{c\cdot \nu(\A,\B)^2
\cdot \log(\card(\FM))}{\sqrt{t}} \Big)  \; \leq \; \ts\frac{1}{t},$$
as needed.
\hfill \BlackBox

\paragraph{Proof of Lemma~\ref{lem:deltaprime} (ii).}


Note that for each $i\in[t]$ and $\j\in\JM$, we have
\begin{equation}
f_{\j}(\D_i) -\EB[f_{\j}(\D_i)] 
\; = \; f_{\j}(\M_i) 
\; = \; \s_i^T (\B \C_{\j}^T \A^T) \s_i - \tr(\B \C_{\j}^T \A^T),
\end{equation}
which is a centered sub-Gaussian quadratic form. Due to the bound~\eqref{psinormquadform}, we have
\begin{equation}
\Big\|f_{\j}(\D_i)- \EB[f_{\j}(\D_i)] \Big\|_{\psi_1} 
\; \leq \; c\, \nu(\A,\B).
\end{equation}
Furthermore, this can be combined with a standard concentration bound for sums of independent sub-exponential random variables (Lemma~\ref{lem:bernstein}) to show that for any $r\geq 0$,
\begin{equation}
\PB\bigg( \Big| \ts\frac{1}{t} \tsum_{i=1}^t f_{\j}(\D_i) -\EB[f_{\j}(\D_i)] \Big| 
\; \geq \; \ r\,\nu(\A,\B) \bigg) \ \leq \ 2\exp\Big(-c \cdot t \cdot \min(r^2,r)\Big).
\end{equation}
Hence, taking a union bound over all $\j$ gives
\begin{equation}\label{tempunion}
\PB\Big( \Delta_t''(\S) \geq \ r\, \nu(\A,\B)\Big)
\; \leq \; 2\exp\Big(\log(\card(\FM))-c \cdot t \cdot \min(r^2,r)\Big).
\end{equation}
Regarding the choice of $r$, note that by Assumption~\ref{assumptions}, we have $\ts\frac{1}{\sqrt{t}}\sqrt{\log(\card(\FM))}\lesssim 1$. It follows that there is a sufficiently large absolute constant $c_1>0$ such that if we put 
$$r=\ts\frac{c_1}{\sqrt{t}}\sqrt{\log(\card(\FM))},$$
 then
$$c\, t\, \min(r^2,r) \geq 2 \log(\card(\FM)),$$
where $c$ is the same as in the bound~\eqref{tempunion}. In turn, this implies
\begin{equation}
\begin{split}
\PB\Big( \Delta_t''(\S) 
\,  \geq \, \ts \ts\frac{c_1}{\sqrt{t}}\sqrt{\log(\card(\FM))}\cdot \nu(\A,\B) \Big)
\; & \leq \; 2\exp(-\log(\card(\FM)) \; = \frac{1}{dd'},
\end{split}
\end{equation}
as desired.
\hfill \BlackBox

\section{Proof of Propositions~\ref{prop:gaussian} and~\ref{prop:bootstrap} in case (b) (length sampling)}\label{app:case:sampling}
In order to carry out the proof Propositions~\ref{prop:gaussian} and~\ref{prop:bootstrap} in the case of length sampling (Assumption~\ref{assumptions} (b)), there are only two bounds that need to be updated. Namely, we must derive new bounds on $\|f_{\j}(\D_1)-\EB[f_{\j}(\D_1)]\|_{\psi_1}$ and $\|f_{\j}(\D_1)\|_{\psi_1}$ in order to account for the new distributional assumptions in case (b). Both of the new bounds will turn out to be of order $\|\A\|_F\|\B\|_F$, and consequently, the result of the  propositions in case (b) will have the same form as in case (a), but with $\|\A\|_F\|\B\|_F$ replacing $\nu(\A,\B)$.

To derive the bound on $\|f_{\j}(\D_1)-\EB[f_{\j}(\D_1)]\|_{\psi_1}$, first note that
\begin{equation*}
\begin{split}
|\EB[f_{\j}(\D_1)]|& = |\tr(\C_{\j}^T\A^T\B)| \\[0.2cm]
&\leq \|\A\|_F\|\B\|_F.
\end{split}
\end{equation*}
Consequently,
\begin{equation}
\begin{split}
\|f_{\j}(\D_1)-\EB[f_{\j}(\D_1)]\|_{\psi_1} &\leq \|f_{\j}(\D_1)\|_{\psi_1}+\|\EB[f_{\j}(\D_1)]\|_{\psi_1}\\[0.2cm]
&\leq  \|f_{\j}(\D_1)\|_{\psi_1}+c\|\A\|_F\|\B\|_F.
\end{split}
\end{equation}
Hence, it remains to show that $ \|f_{\j}(\D_1)\|_{\psi_1}\leq c\|\A\|_F\|\B\|_F$, which is the content of Lemma~\ref{lemma:lengthpsi1} below.
\hfill \BlackBox

\begin{lemma}\label{lemma:lengthpsi1}
If $\S$ is generated by length sampling with the probabilities in line~\eqref{eqn:lengthsampling}, then for  any $\j\in\mathscr{J}$, we have the bound
\begin{equation}
\big\| f_{\j}(\D_1)\big\|_{\psi_1} \leq 2\|\A\|_F\|\B\|_F.
\end{equation}
\end{lemma}

\proof 
By the definition of the $\psi_1$-Orlicz norm, it suffices to find a value of $r>0$ so that $\EB\big[\exp\big(|f_{\j}(\D_1)|/r\big)\big] $ is at most 2.
Due to the Cauchy-Schwarz inequality, the non-zero length-sampling probabilities $p_l$ satisfy
$$\frac{1}{p_l}\leq \frac{\sqrt{\sum_{j=1}^n \|\e_j^T\A\|_2^2}\sqrt{\sum_{j=1}^n \|\e_j^T\B\|_2^2}}{\|\e_l^T\A\|_2\|\e_l^T\B\|_2}=\frac{\|\A\|_F\|\B\|_F}{\|\e_l^T\A\|_2\|\e_l^T\B\|_2}.$$
Consequently, for each $r>0$ we have
\begin{equation*}
\begin{split}
\EB\Big[\exp\Big(\ts\frac{|f_{\j}(\D_1)|}{r}\Big)\Big] 
&= \sum_{l\in[n]: \; p_l>0}p_l \cdot \exp\Big(\ts\frac{1}{r}\big| f_{\j}(\ts\frac{1}{p_l} \A^T \e_l \e_l^T \B)\big|\Big)\\[0.2cm]
&\leq \; \max_{l\in[n]: \; p_l>0} \;
\exp\Big(\ts\frac{1}{r}\ts\frac{1}{p_l}\big| f_{\j}( \A^T \e_l \e_l^T \B)\big|\Big)\\[0.2cm]
&\leq \; \max_{l\in[n]}\; \exp\Big(\ts\frac{1}{r}\ts\|\A\|_F\|\B\|_F\big|\frac{\e_l^T \B}{\|\e_l^T\B\|_2} \C_{\j}^T\frac{\A^T \e_l}{\|\A^T\e_l\|_2}\big|\Big)\\[0.2cm]
&\leq \; \exp\Big(\ts\frac{1}{r}\ts\|\A\|_F\|\B\|_F \ \|\C_{\j}\|_2\Big)\\[0.2cm]
&=\;  \exp\Big(\ts\frac{1}{r}\ts\|\A\|_F\|\B\|_F \Big).
\end{split}
\end{equation*}
Hence, if we take $r=2\|\A\|_F\|\B\|_F$, then the right hand side is at most $e^{1/2}\leq 2$.
\hfill\BlackBox

\section{Proof of Propositions~\ref{prop:gaussian} and~\ref{prop:bootstrap} in case (c) (SRHT)}\label{app:case:srht}
The steps needed to extend the propositions in the case of SRHT matrices follows the same pattern as in case (b). However, there is a small subtlety insofar as all of the analysis is done conditionally on the matrix of signs $\D_n^{\circ}$ in the product $\S=\PP_n (\ts\frac{1}{\sqrt{n}}\H_n)\D_n^{\circ}$. Hence, it suffices to bound the $\psi_1$ Orlicz norm of $f_{\j}(\D_1)$ conditionally on $\D_n^{\circ}$, as well as the conditional expectation $|\EB[f_{\j}(\D_1)|\D_n^{\circ}]$. Regarding the conditional expectation, it can be checked that $\EB[\S^T\S|\D_n^{\circ}]=\I_n$, and it follows that 
$$|\EB[f_{\j}(\D_1)|\D_n^{\circ}]|=|\tr(\B \C_{\j}^T\A^T)|\leq \nu(\A,\B).$$
 Since we are not aware of a standard notation for a conditional Orlicz norm, we define
$$\big\|f_{\j}(\D_1)\big|\D_n^{\circ}\big\|_{\psi_1}:=\inf\Big\{r>0 \ \Big| \ \EB \big[\psi_1\big(|f_{\j}(\D_1)|/r \big)\big| \D_n^{\circ}\big]\leq 1\Big\},$$
which is a random variable, since it is a function of $\D_n^{\circ}$.
The following lemma provides a bound on this quantity, which turns out to be of order $\log(n)\,\nu(\A,\B)$. For this reason, the SRHT case (c) of Propositions~\ref{prop:gaussian} and~\ref{prop:bootstrap} will have the same form as case (a), but with $\log(n)\,\nu(\A,\B)$ replacing $\nu(\A,\B)$. 
\hfill \BlackBox

\begin{lemma}  
If $\S$ is an SRHT matrix, then the following bound holds with probability at least $1-c/n$,
\begin{equation}
 \big\|f_{\j}(\D_1)\big|\D_n^{\circ}\big\|_{\psi_1} \leq c \cdot \log(n) \cdot \nu(\A,\B).
 \end{equation}
\end{lemma}

\begin{proof}
By the definition of the conditional $\psi_1$-Orlicz norm, it suffices to find a value of $r>0$ so that $\EB\big[\exp\big(|f_{\j}(\D_1)|/r\big)\big| \D_n^{\circ}\big] $ is at most 2 (with the stated probability).\\

For an SRHT matrix $\S=\PP\ts\frac{1}{\sqrt{n}}\H_n \D_n^{\circ}$, recall that the rows of $\sqrt{t}\PP$ are sampled uniformly at random from the set $\{\ts\frac{1}{\sqrt{1/n}}\e_1,\dots,\ts\frac{1}{\sqrt{1/n}}\e_n\}$. It follows that

\begin{equation*}
\begin{split}
\EB\Big[\exp\Big(\ts\frac{|f_{\j}(\D_1)|}{r}\Big)\Big| \D_n^{\circ}\Big]  
&=\EB\Big[\exp\Big( \ts\frac{1}{r}\big| \tr(\C_{\j}^T\A^T\s_1\s_1^T\B)\big|\Big)\Big|\D_n^{\circ}\Big]\\[0.2cm]
&=\displaystyle\frac{1}{n}\displaystyle\sum_{l=1}^n  \exp\Big(\ts\frac{1}{r}\big|\e_l^T \big(\H_n \D_n^{\circ}\big)\B \C_{\j}^T \A^T \big(\D_n^{\circ}\H_n^T \big)\e_l\big|\Big).
\end{split}
\end{equation*}
Next, let $\boldsymbol\ve_l\in\RB^n$ be the $l$th row of $\H_n\D_n^{\circ}$, which gives
\begin{equation}\label{srhtorliczexpand}
\begin{split}
\EB\Big[\exp\Big(\ts\frac{|f_{\j}(\D_1)|}{r}\Big)\Big| \D_n^{\circ}\Big]  
&=\frac{1}{n}\displaystyle\sum_{l=1}^n  \exp\Big(\big|\ts\frac{\boldsymbol\varepsilon_l^T (\B\C_{\j}^T\A^T )\boldsymbol\varepsilon_l }{r} \big|\Big)\\[0.2cm]
&\leq \exp\bigg(\max_{l\in [n]} \big|\ts\frac{\boldsymbol\varepsilon_l^T (\B \C_{\j}^T\A^T)\boldsymbol\varepsilon_l }{r} \big|\bigg).
\end{split}
\end{equation}
 Recalling that $\D_n^{\circ}=\diag(\boldsymbol\varepsilon)$ where $\boldsymbol\varepsilon\in\RB^n$ is a vector of i.i.d.~Rademacher variables, and that  all entries of $\H_n$ are $\pm 1$,  it follows that $\boldsymbol \varepsilon_l$  has the same distribution as $\boldsymbol \varepsilon$ for each $l\in[n]$. %
Consequently, each quadratic form $\boldsymbol\varepsilon_l^T (\B\C_{\j}^T\A^T)\boldsymbol\varepsilon_l$ concentrates around $\tr(\B\C_{\j}^T\A^T)$, and we can use a union bound to control the maximum of these quadratic forms. 
Note also that the matrix $\B\C_{\j}^T\A^T$ is rank-1, and so
$\|\B\C_{\j}^T\A^T\|_2^2=\|\B\C_{\j}^T\A^T\|_F^2,$
Hence, by choosing the parameter $u$ to be proportional to $\log(n)\cdot  \|\B\C_{\j}^T\A^T\|_{F}$ in the Hanson-Wright inequality (Lemma~\ref{lem:hansonwright}), and using a union bound, there is an absolute constant $c>0$ such that 
\begin{equation}
\PB\bigg(\max_{l\in[n]}\ \big| \boldsymbol\varepsilon_l^T (\B\C_{\j}^T\A^T)\boldsymbol\varepsilon_l \big| \ \geq \ \tr(\B\C_{\j}^T\A^T)+c\log(n)\|\B\C_{\j}^T\A^T\|_F\bigg) \ \leq \ \ts\frac{c}{n}.
\end{equation}
Furthermore, noting that $\tr(\B\C_{j}^T\A^T)$ and $ \|\B\C_{\j}^T\A^T\|_F$ are both at most $\nu(\A,\B)$, we have
\begin{equation}
\PB\bigg(\max_{l\in[n]} \ \big| \boldsymbol\varepsilon_l^T (\B\C_{\j}^T\A^T)\boldsymbol\varepsilon_l \big| \ \geq \ 2c\log(n)\nu(\A,\B)\bigg) \ \leq \ \ts\frac{c}{n}.
\end{equation}
Finally, this means that if we take $r=4c\log(n)\nu(\A,\B)$ in the bound~\eqref{srhtorliczexpand}, then the event
$$ \EB\big[\exp\big(|f_{\j}(\D_1)|/r\big)\big| \D_n^{\circ}\big]  \leq e^{1/2}$$
holds with probability at least $1-\ts\frac{c}{n}$, which completes the proof, since $e^{1/2}\leq 2$.
\end{proof}

\section{Technical Lemmas} \label{app:technical}

\begin{lemma}[Facts about Orlicz norms]\label{lem:orlicz}
Orlicz norms have the following properties, where $c,c_1,$ and $c_2$ are positive absolute constants.
	\begin{enumerate}
		\item For any random variable $X$, and any $p\geq 1$,
		\begin{equation}\label{lppsi1}
		\|X\|_p\leq c p \|X\|_{\psi_1}
		\end{equation}
		\begin{equation}\label{lppsi2}
		\|X\|_p\leq  c\sqrt{p}\|X\|_{\psi_2}
		\end{equation}
		\begin{equation}\label{psi1psi2}
		\|X\|_{\psi_1}\leq c\|X\|_{\psi_2}.
		\end{equation}
		\item If $X\sim \NM(0,\sigma^2)$, then $\|X\|_{\psi_2}\leq c \,\sigma$.
		
		\item \label{partmax}
		Let $p\geq 1$. For any sequence of random variables $X_1,\dots,X_d$,
		$$\Big\|\max_{1\leq j\leq d} X_j\Big\|_{p} \leq d^{1/p}\max_{1\leq j\leq d} \|X_j\|_{p}$$
		and
		$$\Big\|\max_{1\leq j\leq d} X_j\Big\|_{\psi_1} \leq c\log(d)\max_{1\leq j\leq d} \|X_j\|_{\psi_1}$$
		\item \label{orlicztail} Let $X$ be any random variable. Then, for any $x>0$ and $p\geq 1$, we have
		
		\begin{equation*}
		\PB\big(|X|\geq x\big)\leq \Big(\ts\frac{\|X\|_p}{x}\Big)^p.
		\end{equation*}
		and
		\begin{equation*}
		\PB\big(|X|\geq x\big)\leq c_1 e^{-c_2x/\|X\|_{\psi_1}}.
		\end{equation*}
	\end{enumerate}
\end{lemma}

\begin{proof}
In part 1, line~\eqref{lppsi1} follows from line 5.11 of~\citet{vershyninIntro}, line~\eqref{lppsi2} follows from definition 5.13 of~\citet{vershyninIntro}, and line~\eqref{psi1psi2} follows from p.94 of~\citet{vaartWellner}. Next, part 2 follows from the definition of the $\psi_2$-Orlicz norm and the moment generating function for $\NM(0,\sigma^2)$. Part 3 is due to Lemma 2.2.2 of~\citet{vaartWellner}. Lastly, part 4 follows from Markov's inequality and line 5.14 of~\citet{vershyninIntro}.
\end{proof}

\begin{lemma}[Rosenthal's inequality with best constants]\label{lem:rosenthal}
	Fix any number $p>2$. Let $Y_1,\dots,Y_t$ be independent random variables with $\EB[Y_i]=0$ 
	and $\EB \big[|Y_i|^p\big]<\infty$ for all $1\leq i\leq t$. Then,
	\begin{equation}
	\big\|\tsum_{i=1}^t Y_i\big\|_p\leq c\big(\ts\frac{p}{\log(p)}\big) \cdot \max\bigg\{ \big\|\tsum_{i=1}^t Y_i\big\|_2 \, , \, \big(\tsum_{i=1}^t \big\|Y_i\|_p^p\big)^{1/p}\bigg\}.
	\end{equation}
\end{lemma}

\begin{proof}
See the paper \citet{rosenthalbest}.
The statement above differs slightly from the Theorem 4.1 in the paper~\citet{rosenthalbest}, which requires symmetric random variables, but the 
 remark on p.247 of that paper explains why the variables $Y_1,\dots,Y_t$ need not be symmetric as long as they have mean 0.
\end{proof}

\begin{lemma}[Hanson-Wright inequality]\label{lem:hansonwright}
	Let $\x=(X_1, \dots, X_n)$ be a vector of independent sub-Gaussian random variables with $\EB [X_j]=0$, 
	and $\|X_j\|_{\psi_2}\leq \kappa$ for all $1\leq j\leq n$. 
	Also, let $\H \in \RB^{n\times n}$ be any fixed non-zero matrix. 
	Then, there is an absolute constant $c>0$ such that for any $u\geq 0$,
	\begin{equation}
	\PB \Big(\big| \x^T \H \x - \EB[\x^T \H \x]\big| \ \geq \ u\Big) 
	\; \leq \; 
	2\exp\Big(-c \cdot \min\Big\{\ts\frac{u^2}{\kappa^2\|\H\|_F^2},
	\, \ts\frac{u}{\kappa\|\H\|_2}\Big\}\Big).
	\end{equation}
\end{lemma}

\begin{proof}
See the paper \citet{rudelsonVershynin2013}.
\end{proof}

\begin{lemma}[Bernstein inequality for sub-exponential variables]\label{lem:bernstein}
Let $Y_1,\dots,Y_t$ be independent random variables with $\EB[Y_i]=0$ and  $\|Y_i\|_{\psi_1} \leq \kappa$ for all $1\leq i\leq t$. Then, there is an absolute constant $c>0$, such that for any $u\geq 0$,
\begin{equation}
\PB\bigg(\Big|\ts\frac{1}{t}\sum_{i=1}^t Y_i \Big| \geq \kappa\cdot u \bigg) \; \leq \; 2\exp\Big( -c\cdot t\cdot \min(u^2,u)\Big).
\end{equation}
\end{lemma}
\begin{proof}
See Proposition 16 in~\cite{vershyninIntro}.
\end{proof}

\begin{lemma}[\citep{cckIncreasing}]\label{lem:jkbounds} 
If $\eta$ is a non-negative random variable, and there are numbers $a,b>0$ such that
$$\PB(\eta>x)\leq a e^{-x/b},$$
for all $x>0$, then the following bound holds for all $r>0$,
$$\EB[\eta^3\cdot 1\{\eta>r\}]\leq 6a(r+b)^3e^{-r/b}.$$
\end{lemma}

\begin{proof}
See Lemma 6.6 in~\citet{cckIncreasing}.
\end{proof}

\paragraph{Remark.}	The following lemma may be of independent interest, since it provides an explicit bound on the $\psi_1$-Orlicz norm of a centered sub-Gaussian quadratic form. Although this bound follows from the Hanson-Wright inequality, we have not seen it stated in the literature.

\begin{lemma}\label{lem:quadnorm}
  Let $\x=(X_1,\dots,X_n)$, be independent random variables satisfying 
	$\EB [X_j]=0$ and $\|X_j\|_{\psi_2}\leq \kappa$ for all $1\leq j\leq n$. 
	Also, let $\H\in\RB^{n\times n}$ be a non-zero fixed matrix. 
	Then, there is an absolute constant $c>0$ such that
	\begin{equation*}
	\Big\|\x^T \H \x -  \EB[\x^T \H \x] \Big\|_{\psi_1} 
	\; \leq \; c \kappa^2 \ts\frac{\|\H\|_F^2}{\|\H\|_2}.
	\end{equation*}
\end{lemma}

\begin{proof}
Define the random variable $Q:=\x^T \H\x -  \EB [\x^T \H \x]$. 
By the definition of the $\psi_1$-Orlicz norm, 
it suffices to find a value $r>0$ such that
$\EB \big[\exp(|Q|/r)\big] \leq 2$.
Using the tail-sum formula, and the change of variable $v=e^{u/r}$, we have
\begin{equation*}
\begin{split}
\EB \big[\exp(|Q|/r)\big] &\leq 1+ \int_1^{\infty} \PB(\exp(|Q|/r)> v)dv\\[0.2cm]
&=1+\ts\frac{1}{r}\displaystyle \int_0^{\infty} \PB(|Q|> u)\cdot e^{u/r}du.
\end{split}
\end{equation*}
Next, we employ the Hanson-Wright inequality (Lemma~\ref{lem:hansonwright}). By considering the ``threshold'' $ u^*:=\kappa^2\ts\frac{\|\H\|_F^2}{\|\H\|_2}$, it is helpful to note that the quantities in the exponent of the Hanson-Wright inequality satisfy  $\ts\frac{u^2}{\kappa^4 \|\H\|_F^2}\leq \ts\frac{u}{\kappa^2\|\H\|_2}$ if and only if $u\leq u^*$. Hence,
\begin{equation*}
\small
\begin{split}
\EB\big[\exp(|Q|/r)\big] &\leq 1+\ts\frac{1}{r}\displaystyle \int_0^{\infty} \exp\Big\{-c\min\big(\ts\frac{u^2}{\kappa^4\|\H\|_F^2}\, ,  \ts\frac{u}{\kappa^2\|\H\|_2}\big)\Big\}\cdot e^{u/r}du\\[0.2cm]
&\leq 1+\ts\frac{1}{r}\displaystyle \int_0^{u^*}  e^{u/r}du + \ts\frac{1}{r}\displaystyle \int_{u^*}^{\infty} \exp\Big\{\!-u\big(\ts\frac{c }{\kappa^2\|\H\|_2}-\ts\frac{1}{r}\big)\Big\}du.
\end{split}
\end{equation*}
Evaluating the last two integrals directly, if we let $ C':=\ts\frac{c }{\kappa^2\|\H\|_2}-\ts\frac{1}{r}$ and choose $r$ so that $C'>0$, then
\begin{equation*}
\begin{split}
\EB\big[\exp(|Q|/r)\big] 
&\leq e^{u^*/r}+\ts\frac{1}{r C'}e^{-u^* C'},\\[0.2cm]
& \leq  e^{u^*/r}+\ts\frac{1}{ \ts\frac{c \cdot r}{\kappa^2 \|\H\|_2}-1}.
\end{split}
\end{equation*}
\normalsize
Note that the condition $C'>0$ means that it is necessary to have $r>\ts\frac{1}{c}\kappa^2\|\H\|_2$. To finish the argument, we further require that $r$ is large enough so that (say)
\begin{equation}\label{constraints}
\ts\frac{u^*}{r} \leq \ts\frac{1}{4} \ \ \  \text{ and }  \ \ \ \   \ts\frac{c\cdot r}{\kappa^2\|\H\|_2} \geq 3,
\end{equation}
which ensures
\begin{equation*}
\EB\big[\exp(|Q|/r)\big] \leq e^{1/4}+\ts\frac{1}{2} < 2,
\end{equation*}
as desired.
Note that the constraints~\eqref{constraints} are the same as
\begin{equation*}
r \geq 4\kappa^2 \ts\frac{\|\H\|_F^2}{\|\H\|_2} \ \ \ \ \text{ and } \ \ \ \ \  r\geq \ts\frac{3}{c} \kappa^2\|\H\|_2.
\end{equation*}
Due to the basic fact that $\|\H\|_2\leq \|\H\|_F$ for all matrices $\H$, it follows that whenever $r\geq \max(4,\ts\frac{3}{c})\kappa^2\ts\frac{\|\H\|_F^2}{\|\H\|_2}$, we have $\EB\big[\exp(|Q|/r)\big] < 2$. 
\end{proof}

\paragraph{Remark.} The following lemma is a basic fact about the $d_{\text{LP}}$ metric, but may not be widely known, and so we give a proof. Recall also that we use the generalized inverse $F_V^{-1}(\alpha):=\inf\{z\in\RB \, |\, F_V(z)\geq \alpha\}$, where $F_V$ denotes the c.d.f.~of $V$.

\begin{lemma}\label{LPquantile}
	Fix $\alpha\in(0,1/2)$ and suppose there is some $\epsilon\in(0,\alpha)$ such that random variables $U$ and $V$ satisfy
	$$d_{\text{\emph{LP}}}(\mathcal{L}(U),\mathcal{L}(V))\leq \epsilon.$$
	Then, the quantiles of $U$ and $V$  satisfy
	\begin{equation}\label{quantilereln}
\big|F_{U}^{-1}(1-\alpha)-F_V^{-1}(1-\alpha)\big| \ \leq \ \psi_{\alpha}(\epsilon),
\end{equation}
where the right side is defined as
$$ \psi_{\alpha}(\epsilon):=F_U^{-1}(1-\alpha+\epsilon)-F_U^{-1}(1-\alpha-\epsilon)+\epsilon.$$
\end{lemma}

\begin{proof}
Consider the L\'evy metric, defined as
\begin{equation*}
\small
d_{\text{L}}(\mathcal{L}(U),\mathcal{L}(V)):=\inf\Big\{ \epsilon>0 \, \Big|\, F_U(x-\epsilon)-\epsilon\leq F_V(x)\leq F_U(x+\epsilon)+\epsilon \text{ for all } x\in\RB \Big\}.
\end{equation*}
It is a fact that this metric is always dominated by the $d_{\text{LP}}$ metric in the sense that
$$d_{\text{L}}(\mathcal{L}(U),\mathcal{L}(V)) \ \leq \  d_{\text{LP}}(\mathcal{L}(U),\mathcal{L}(V)),$$ for all scalar random variables $U$ and $V$~\citep[p.36]{huberbook}. Based on the definition of the $d_{\text{L}}$ metric, it is straightforward to check that the following inequalities hold under the assumption of the lemma,
$$F_U^{-1}(1-\alpha-\epsilon) - \epsilon \ \leq \  F_V^{-1}(1-\alpha) \ \leq \ F_U^{-1}(1-\alpha+\epsilon)+\epsilon.$$
(Specifically, consider the choices $x=F_V^{-1}(1-\alpha)$ and $x=F_U^{-1}(1-\alpha+\epsilon)+\epsilon$.)
Next, if we subtract $F_U^{-1}(1-\alpha)$ from each side of the inequalities above, and note that $F_U^{-1}(\cdot)$ is non-decreasing, it follows that if we put $a=F_U^{-1}(1-\alpha+\epsilon)-F_U^{-1}(1-\alpha)+\epsilon$ and $b=F_U^{-1}(1-\alpha)-F_U^{-1}(1-\alpha-\epsilon)+\epsilon$, then
$$ |F_V^{-1}(\alpha) -F_U^{-1}(\alpha)| \ \leq \ \max\{a,b\} \ \leq \ \psi_{\alpha}(\epsilon),$$
as needed.
\end{proof}

\begin{lemma}\label{lem:verify}
	Under Assumption~\ref{assumptions} (a), the quantity 
	\begin{equation*}
	\delta_0=  t^{-1/8} \log^{1/2}(d)\nu(\A,\B)^{3/4}
	\end{equation*}
	satisfies conditions~\eqref{firstconstraint} and~\eqref{solvedelta}.
\end{lemma}

\begin{proof}
 Consider the number
\begin{equation*}
\delta_1(r):=\ts\frac{ \log^2(d)\nu(\A,\B)}{ \sqrt{t}}\cdot r,
\end{equation*}
where $r\geq 1$ is a free parameter to be adjusted. Based on the bound~\eqref{jkbound}, it is easy to check that plugging $\delta_1(r)$ into $K_t(\cdot)$ and $J_t(\cdot)$ leads to
\begin{equation*}
\begin{split}
K_t(\delta_1(r))+J_t(\delta_1(r)) &\leq c\cdot \nu(\A,\B)^3 \cdot \log(d)^3\cdot \, r^3\cdot \exp(-r/c)
\end{split}
\end{equation*}
and if we take  $r\geq c \log(\log(d)^4)$, then
\begin{equation*}
K_t(\delta_1(r))+J_t(\delta_1(r)) \leq c\, \nu(\A,\B)^3,
\end{equation*}
as desired in~\eqref{firstconstraint}. Hence, as long as there is a choice of $r$ satisfying 
\begin{equation*}
r\geq c \log(\log(d)^4) \ \ \ \ \ \text{ and }  \ \ \  \ \  \delta_1(r)=\delta_0,
\end{equation*}
then $\delta_1(r)$ will satisfy both of the desired constraints~\eqref{firstconstraint} and~\eqref{solvedelta}.
Solving the equation $\delta_1(r)=\delta_0$ gives
\begin{equation*}
r= t^{3/8}\cdot \log^{-3/2}(d)\cdot \nu(\A,\B)^{-1/4},
\end{equation*}
and then the condition $r\geq c \log(\log(d)^4)$ is the same as
\begin{equation}\label{loglogbound}
\begin{split}
t &\geq \Big(c \,\nu(\A,\B)^{1/4} \log(d)^{3/2} \cdot \log(\log(d)^4)\Big)^{8/3}\\[0.2cm]
&=c\, \nu(\A,\B)^{2/3} \log(d)^{4} \cdot \log(\log(d)^4)^{8/3},
\end{split}
\end{equation}
which holds under Assumption~\ref{assumptions} (a).
\end{proof}
%


\bibliography{sketch_mult_bib_rev}

\begin{thebibliography}{51}
\providecommand{\natexlab}[1]{#1}
\providecommand{\url}[1]{\texttt{#1}}
\expandafter\ifx\csname urlstyle\endcsname\relax
  \providecommand{\doi}[1]{doi: #1}\else
  \providecommand{\doi}{doi: \begingroup \urlstyle{rm}\Url}\fi

\bibitem[Ailon and Chazelle(2006)]{ailon2006}
N.~Ailon and B.~Chazelle.
\newblock Approximate nearest neighbors and the fast {J}ohnson-{L}indenstrauss
  transform.
\newblock In \emph{Annual ACM Symposium on Theory of Computing (STOC)}, 2006.

\bibitem[Ailon and Liberty(2009)]{ailon2009fast}
N.~Ailon and E.~Liberty.
\newblock Fast dimension reduction using {R}ademacher series on dual {BCH}
  codes.
\newblock \emph{Discrete \& Computational Geometry}, 42\penalty0 (4):\penalty0
  615--630, 2009.

\bibitem[Ar et~al.(1993)Ar, Blum, Codenotti, and Gemmell]{blum}
S.~Ar, M.~Blum, B.~Codenotti, and P.~Gemmell.
\newblock Checking approximate computations over the reals.
\newblock In \emph{Annual ACM Symposium on Theory of Computing (STOC)}, 1993.

\bibitem[Avron et~al.(2010)Avron, Maymounkov, and Toledo]{blendenpik}
H.~Avron, P.~Maymounkov, and S.~Toledo.
\newblock Blendenpik: Supercharging lapack's least-squares solver.
\newblock \emph{SIAM Journal on Scientific Computing}, 32\penalty0
  (3):\penalty0 1217--1236, 2010.

\bibitem[Boutsidis and Gittens(2013)]{boutsidis2013}
C.~Boutsidis and A.~Gittens.
\newblock Improved matrix algorithms via the subsampled randomized hadamard
  transform.
\newblock \emph{SIAM Journal on Matrix Analysis and Applications}, 34\penalty0
  (3):\penalty0 1301--1340, 2013.

\bibitem[Brezinski and Zaglia(2013)]{brezinski}
C.~Brezinski and M.~R. Zaglia.
\newblock \emph{Extrapolation methods: theory and practice}.
\newblock Elsevier, 2013.

\bibitem[Chang and Lin(2011)]{libsvm}
C.-C. Chang and C.-J. Lin.
\newblock {LIBSVM}: a library for support vector machines.
\newblock \emph{ACM Transactions on Intelligent Systems and Technology (TIST)},
  2\penalty0 (3):\penalty0 27, 2011.
\newblock URL \url{http://www.csie.ntu.edu.tw/~cjlin/libsvmtools/datasets/}.

\bibitem[Chang et~al.(2016)Chang, Zhou, Zhou, and Wang]{chang_biometrics}
J.~Chang, W.~Zhou, W.-X. Zhou, and L.~Wang.
\newblock Comparing large covariance matrices under weak conditions on the
  dependence structure and its application to gene clustering.
\newblock \emph{Biometrics}, 2016.

\bibitem[Chen(2018)]{chen_ustat}
X.~Chen.
\newblock Gaussian and bootstrap approximations for high-dimensional
  u-statistics and their applications.
\newblock \emph{The Annals of Statistics}, 46\penalty0 (2):\penalty0 642--678,
  2018.

\bibitem[Chernozhukov et~al.(2013)Chernozhukov, Chetverikov, and
  Kato]{cckMultiplier}
V.~Chernozhukov, D.~Chetverikov, and K.~Kato.
\newblock Gaussian approximations and multiplier bootstrap for maxima of sums
  of high-dimensional random vectors.
\newblock \emph{The Annals of Statistics}, 41\penalty0 (6):\penalty0
  2786--2819, 2013.

\bibitem[Chernozhukov et~al.(2014)Chernozhukov, Chetverikov, and
  Kato]{cckSuprema}
V.~Chernozhukov, D.~Chetverikov, and K.~Kato.
\newblock Gaussian approximation of suprema of empirical processes.
\newblock \emph{The Annals of Statistics}, 42\penalty0 (4):\penalty0
  1564--1597, 2014.

\bibitem[Chernozhukov et~al.(2015)Chernozhukov, Chetverikov, and Kato]{cckAnti}
V.~Chernozhukov, D.~Chetverikov, and K.~Kato.
\newblock Comparison and anti-concentration bounds for maxima of {G}aussian
  random vectors.
\newblock \emph{Probability Theory and Related Fields}, 162\penalty0
  (1-2):\penalty0 47--70, 2015.

\bibitem[Chernozhukov et~al.(2016)Chernozhukov, Chetverikov, and
  Kato]{cckIncreasing}
V.~Chernozhukov, D.~Chetverikov, and K.~Kato.
\newblock Empirical and multiplier bootstraps for suprema of empirical
  processes of increasing complexity, and related {G}aussian couplings.
\newblock \emph{Stochastic Processes and their Applications}, 2016.

\bibitem[Chernozhukov et~al.(2017)Chernozhukov, Chetverikov, and Kato]{cckCLT}
V.~Chernozhukov, D.~Chetverikov, and K.~Kato.
\newblock Central limit theorems and bootstrap in high dimensions.
\newblock \emph{The Annals of Probability}, 45\penalty0 (4):\penalty0
  2309--2352, 2017.

\bibitem[Clarkson and Woodruff(2013)]{clarkson2013low}
K.~L. Clarkson and D.~P. Woodruff.
\newblock Low rank approximation and regression in input sparsity time.
\newblock In \emph{Annual ACM Symposium on theory of computing (STOC)}, 2013.

\bibitem[Dasarathy et~al.(2015)Dasarathy, Shah, Bhaskar, and
  Nowak]{covariancesketching}
G.~Dasarathy, P.~Shah, B.~Narayan Bhaskar, and R.~D. Nowak.
\newblock Sketching sparse matrices, covariances, and graphs via tensor
  products.
\newblock \emph{IEEE Transactions on Information Theory}, 61\penalty0
  (3):\penalty0 1373--1388, 2015.

\bibitem[Demmel et~al.(2007)Demmel, Dumitriu, Holtz, and Kleinberg]{demmel2007}
J.~Demmel, I.~Dumitriu, O.~Holtz, and R.~Kleinberg.
\newblock Fast matrix multiplication is stable.
\newblock \emph{Numerische Mathematik}, 106\penalty0 (2):\penalty0 199--224,
  2007.

\bibitem[Dixon(1983)]{dixon1983}
J.~D. Dixon.
\newblock Estimating extremal eigenvalues and condition numbers of matrices.
\newblock \emph{SIAM Journal on Numerical Analysis}, 20\penalty0 (4):\penalty0
  812--814, 1983.

\bibitem[Drineas and Mahoney(2016)]{drineas2016randnla}
P.~Drineas and M.~W. Mahoney.
\newblock {RandNLA}: randomized numerical linear algebra.
\newblock \emph{Communications of the ACM}, 59\penalty0 (6):\penalty0 80--90,
  2016.

\bibitem[Drineas et~al.(2006{\natexlab{a}})Drineas, Kannan, and
  Mahoney]{drineas06fastmonte1}
P.~Drineas, R.~Kannan, and M.~W. Mahoney.
\newblock Fast {M}onte {C}arlo algorithms for matrices {I}: Approximating
  matrix multiplication.
\newblock \emph{SIAM Journal on Computing}, 36\penalty0 (1):\penalty0 132--157,
  2006{\natexlab{a}}.

\bibitem[Drineas et~al.(2006{\natexlab{b}})Drineas, Mahoney, and
  Muthukrishnan]{drineas2006sampling}
P.~Drineas, M.~W. Mahoney, and S.~Muthukrishnan.
\newblock Sampling algorithms for $\ell_2$ regression and applications.
\newblock In \emph{Annual ACM-SIAM Symposium on Discrete Algorithm (SODA)},
  2006{\natexlab{b}}.

\bibitem[Drineas et~al.(2008)Drineas, Mahoney, and
  Muthukrishnan]{drineas2008cur}
P.~Drineas, M.~W. Mahoney, and S.~Muthukrishnan.
\newblock Relative-error {CUR} matrix decompositions.
\newblock \emph{SIAM Journal on Matrix Analysis and Applications}, 30\penalty0
  (2):\penalty0 844--881, September 2008.

\bibitem[Drineas et~al.(2011)Drineas, Mahoney, Muthukrishnan, and
  Sarl{\'o}s]{drineas2011faster}
P.~Drineas, M.~W. Mahoney, S.~Muthukrishnan, and T.~Sarl{\'o}s.
\newblock Faster least squares approximation.
\newblock \emph{Numerische Mathematik}, 117\penalty0 (2):\penalty0 219--249,
  2011.

\bibitem[Drineas et~al.(2012)Drineas, Magdon-Ismail, Mahoney, and
  Woodruff]{DMMW12_JMLR}
P.~Drineas, M.~Magdon-Ismail, M.~W. Mahoney, and D.~P. Woodruff.
\newblock Fast approximation of matrix coherence and statistical leverage.
\newblock \emph{Journal of Machine Learning Research}, 13:\penalty0 3475--3506,
  2012.

\bibitem[Frank and Asuncion(2010)]{uci2010}
A.~Frank and A.~Asuncion.
\newblock {UCI} machine learning repository, 2010.
\newblock URL \url{http://archive.ics.uci.edu/ml}.

\bibitem[Freivalds(1979)]{freivalds1979}
R.~Freivalds.
\newblock Fast probabilistic algorithms.
\newblock \emph{Mathematical Foundations of Computer Science}, pages 57--69,
  1979.

\bibitem[Halko et~al.(2011)Halko, Martinsson, and Tropp]{halko2011random}
N.~Halko, P.-G. Martinsson, and J.~A. Tropp.
\newblock Finding structure with randomness: probabilistic algorithms for
  constructing approximate matrix decompositions.
\newblock \emph{SIAM Review}, 53\penalty0 (2):\penalty0 217--288, 2011.

\bibitem[Higham(2002)]{higham2002}
N.~J. Higham.
\newblock \emph{Accuracy and stability of numerical algorithms}.
\newblock SIAM, 2002.

\bibitem[Huber and Ronchetti(2009)]{huberbook}
P.~J. Huber and E.~M. Ronchetti.
\newblock \emph{Robust Statistics}.
\newblock Wiley, 2009.

\bibitem[Johnson and Lindenstrauss(1984)]{johnson1984extensions}
W.~B. Johnson and J.~Lindenstrauss.
\newblock Extensions of {L}ipschitz mappings into a {H}ilbert space.
\newblock \emph{Contemporary mathematics}, 26\penalty0 (189-206), 1984.

\bibitem[Johnson et~al.(1985)Johnson, Schechtman, and Zinn]{rosenthalbest}
W.~B. Johnson, G.~Schechtman, and J.~Zinn.
\newblock Best constants in moment inequalities for linear combinations of
  independent and exchangeable random variables.
\newblock \emph{The Annals of Probability}, pages 234--253, 1985.

\bibitem[Liberty et~al.(2007)Liberty, Woolfe, Martinsson, Rokhlin, and
  Tygert]{liberty2007}
E.~Liberty, F.~Woolfe, P.-G. Martinsson, V.~Rokhlin, and M.~Tygert.
\newblock Randomized algorithms for the low-rank approximation of matrices.
\newblock \emph{Proceedings of the National Academy of Sciences}, 104\penalty0
  (51):\penalty0 20167--20172, 2007.

\bibitem[Lopes(2019)]{Lopes2019}
M.~E. Lopes.
\newblock Estimating the algorithmic variance of randomized ensembles via the
  bootstrap.
\newblock \emph{The Annals of Statistics}, 47\penalty0 (2):\penalty0
  1088--1112, 2019.

\bibitem[Lopes et~al.(2018{\natexlab{a}})Lopes, Lin, and Mueller]{Lopes:2018}
M.~E. Lopes, Z.~Lin, and H.-G. Mueller.
\newblock Bootstrapping max statistics in high dimensions: Near-parametric
  rates under weak variance decay and application to functional data analysis.
\newblock \emph{arXiv:1807.04429}, 2018{\natexlab{a}}.

\bibitem[Lopes et~al.(2018{\natexlab{b}})Lopes, Wang, and Mahoney]{LopesICML}
M.~E. Lopes, S.~Wang, and M.~W. Mahoney.
\newblock Error estimation for randomized least-squares algorithms via the
  bootstrap.
\newblock In \emph{Proceedings of the 35th International Conference on Machine
  Learning (ICML)}, 2018{\natexlab{b}}.

\bibitem[Ma et~al.(2014)Ma, Mahoney, and Yu]{ma2014statistical}
P.~Ma, M.~W. Mahoney, and B.~Yu.
\newblock A statistical perspective on algorithmic leveraging.
\newblock In \emph{International Conference on Machine Learning (ICML)}, 2014.

\bibitem[Magen and Zouzias(2011)]{magen2011low}
A.~Magen and A.~Zouzias.
\newblock Low rank matrix-valued {C}hernoff bounds and approximate matrix
  multiplication.
\newblock In \emph{Annual ACM-SIAM Symposium on Discrete Algorithms (SODA)},
  2011.

\bibitem[Mahoney(2011)]{mahoney2011randomized}
M.~W. Mahoney.
\newblock Randomized algorithms for matrices and data.
\newblock \emph{Foundations and Trends in Machine Learning}, 3\penalty0
  (2):\penalty0 123--224, 2011.

\bibitem[Pagh(2013)]{pagh2013}
R.~Pagh.
\newblock Compressed matrix multiplication.
\newblock \emph{ACM Transactions on Computation Theory}, 5\penalty0
  (3):\penalty0 9, 2013.

\bibitem[Pilanci and Wainwright(2017)]{pilanci2015newton}
M.~Pilanci and M.~J. Wainwright.
\newblock Newton sketch: a near linear-time optimization algorithm with
  linear-quadratic convergence.
\newblock \emph{SIAM Journal on Optimization}, 27\penalty0 (1):\penalty0
  205--245, 2017.

\bibitem[Roosta-Khorasani and Mahoney(2016)]{roosta2016sub}
F.~Roosta-Khorasani and M.~W. Mahoney.
\newblock Sub-sampled {N}ewton methods {II}: local convergence rates.
\newblock \emph{arXiv:1601.04738}, 2016.

\bibitem[Rudelson and Vershynin(2013)]{rudelsonVershynin2013}
M.~Rudelson and R.~Vershynin.
\newblock {H}anson-{W}right inequality and sub-{G}aussian concentration.
\newblock \emph{Electronic Communications in Probability}, 18:\penalty0 9 pp.,
  2013.

\bibitem[Sarl{\'o}s(2006)]{sarlos2006}
T.~Sarl{\'o}s.
\newblock Improved approximation algorithms for large matrices via random
  projections.
\newblock In \emph{Annual IEEE Symposium on Foundations of Computer Science
  (FOCS)}, 2006.

\bibitem[Sidi(2003)]{sidi}
A.~Sidi.
\newblock \emph{Practical Extrapolation Methods: Theory and Applications}.
\newblock Cambridge University Press, 2003.

\bibitem[van~der Vaart and Wellner(1996)]{vaartWellner}
A.~W. van~der Vaart and J.~A. Wellner.
\newblock \emph{Weak Convergence and Empirical Processes}.
\newblock Springer, 1996.

\bibitem[Vershynin(2012)]{vershyninIntro}
R.~Vershynin.
\newblock Introduction to the non-asymptotic analysis of random matrices.
\newblock In \emph{Compressed Sensing, Theory and Applications}. Cambridge
  University Press, 2012.

\bibitem[Wang(2015)]{wang2015practical}
S.~Wang.
\newblock A practical guide to randomized matrix computations with {MATLAB}
  implementations.
\newblock \emph{arXiv:1505.07570}, 2015.

\bibitem[Woodruff(2014)]{woodruff2014sketching}
D.~P. Woodruff.
\newblock Sketching as a tool for numerical linear algebra.
\newblock \emph{Foundations and Trends{\textregistered} in Theoretical Computer
  Science}, 10\penalty0 (1--2):\penalty0 1--157, 2014.

\bibitem[Woolfe et~al.(2008)Woolfe, Liberty, Rokhlin, and Tygert]{woolfe2008}
F.~Woolfe, E.~Liberty, V.~Rokhlin, and M.~Tygert.
\newblock A fast randomized algorithm for the approximation of matrices.
\newblock \emph{Applied and Computational Harmonic Analysis}, 25\penalty0
  (3):\penalty0 335--366, 2008.

\bibitem[Xu et~al.(2016)Xu, Yang, Roosta-Khorasani, R{\'e}, and
  Mahoney]{xu2016sub}
P.~Xu, J.~Yang, F.~Roosta-Khorasani, C.~R{\'e}, and M.~W. Mahoney.
\newblock Sub-sampled {N}ewton methods with non-uniform sampling.
\newblock In \emph{Advances in Neural Information Processing Systems (NIPS)},
  pages 3000--3008, 2016.

\bibitem[Yang et~al.(2016)Yang, Meng, and Mahoney]{yang2015implementing}
J.~Yang, X.~Meng, and M.~W. Mahoney.
\newblock Implementing randomized matrix algorithms in parallel and distributed
  environments.
\newblock \emph{Proceedings of the IEEE}, 104\penalty0 (1):\penalty0 58--92,
  2016.

\end{thebibliography}

\end{document}